\documentclass[sigconf,nonacm]{acmart}

\AtBeginDocument{%
  }

\usepackage{multirow}   % 刚才的 Table 2 用到了跨行合并
\usepackage{booktabs}   % 三线表支持 (toprule, midrule, bottomrule) - acmart通常已内置，但显式调用无害
\usepackage{graphicx}   % 图片插入及 resizebox 缩放表格 - acmart通常已内置

\usepackage{algorithm}   % 提供 algorithm 浮动体
\usepackage{algpseudocode}% 提供算法语句（推荐！比 algorithmic 更简洁）
\newcommand{\inc}[1]{\scriptsize{\textcolor{teal}{(+#1)}}}
\newcommand{\dec}[1]{\scriptsize{\textcolor{red}{(-#1)}}}
\usepackage{xcolor} 
\renewcommand\footnotetextcopyrightpermission[1]{}

\renewcommand\footnotetextcopyrightpermission[1]{}

\begin{document}

\title{Efficient Data Selection for Multimodal Models \\ via Incremental Optimization Utility}

\author{Jinhao Jing}
\affiliation{%
  \institution{ }
  \country{ }}
\email{kawakamitomie56@gmail.com}

\author{Qiannian Zhao\footnotemark[1]}
\affiliation{%
  \institution{ }
  \country{ }}
\email{zqn200285@gmail.com}

\author{Chao Huang}
\affiliation{%
  \institution{ }
  \country{ }}
\email{huangchao200107@gmail.com}

\author{Zhan Su\footnotemark[2]}
\affiliation{%
  \institution{ }
  \country{ }}
\email{zhan.su@umontreal.ca}

%% 脚注文本：放在\maketitle之后，自动出现在第一页左下角（和参考论文完全一致）




\begin{abstract}

The scaling of Large Multimodal Models (LMMs) is constrained by the quality-quantity trade-off inherent in synthetic data. Previous approaches, such as \textit{LLM-as-a-Judge}, have proven their effectiveness in addressing this but suffer from prohibitive computational costs and lack of interpretability. To bridge this gap, we propose \textbf{One-Step-Train (OST)}, a framework that reformulates data selection as an \textit{incremental optimization utility} ranking problem. Instead of relying on semantic heuristics, OST estimates the marginal utility of each sample via a simulated single-step update on a lightweight proxy. Experiments on the Qwen series across multimodal mathematical reasoning benchmarks demonstrate that OST achieves Pareto-optimal efficiency. By selecting the top-\textbf{50\%} subset, OST reduces training costs by \textbf{43\%} (and total time consumption by \textbf{17\%}) while surpassing the strong \textit{LLM-as-a-Judge} baseline by \textbf{1.8} points. Furthermore, under a fixed compute budget, our method using only the top-\textbf{20\%} subset achieves a \textbf{5.6} point gain over \textit{LLM-as-a-Judge}, improves upon heuristic scoring baselines like \textit{DEITA}, and outperforms the Full-SFT baseline by \textbf{8.8} points. Notably, while Full-SFT suffers from performance degradation due to noise, our optimization-grounded approach effectively identifies toxic samples, successfully reversing the negative transfer frequently observed in complex reasoning tasks.

\end{abstract}



%
% Keywords. The author(s) should pick words that accurately describe
% the work being presented. Separate the keywords with commas.
% A "teaser" image appears between the author and affiliation
% information and the body of the document, and typically spans the
% page.

\maketitle

\footnotetext[1]{Contributing equally with the first author.}
\footnotetext[2]{Corresponding author.}

\begin{figure*}[t!]
    \centering
    % Ensure your file path is correct. I kept the original placeholder name.
    \includegraphics[width=1.0\textwidth]{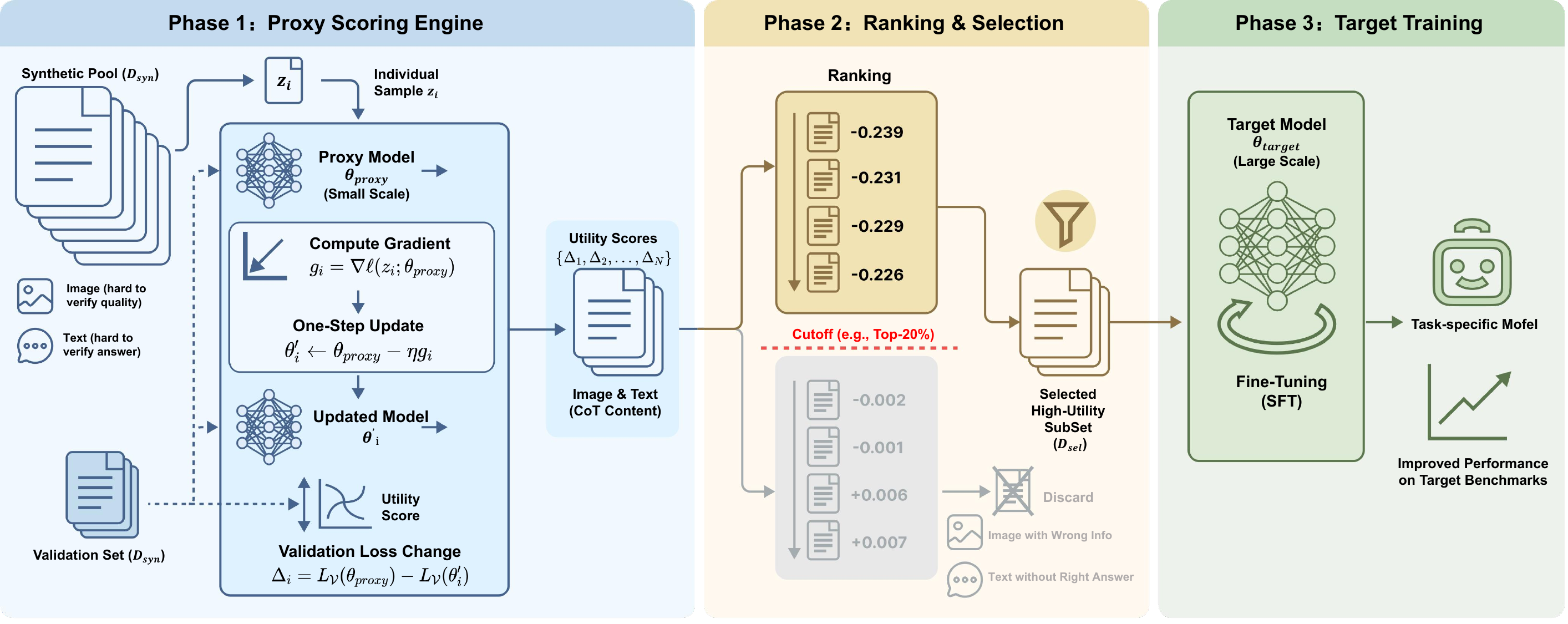}
    \vspace{-1.5em}
    \caption{\textbf{Overview of the One-Step-Train (OST) Pipeline.} The process reformulates selection into three efficient phases: (1) \textbf{Proxy Scoring:} A lightweight proxy model estimates the marginal utility ($\Delta_i$) of each sample by measuring validation loss reduction after a simulated single gradient update. (2) \textbf{Ranking \& Selection:} Samples are ordered by their update utility, filtering a high-value top-$p$\% subset. (3) \textbf{Target Training:} The large target model is fine-tuned exclusively on this subset, effectively reversing negative transfer.}
    \label{fig:ost_pipeline}
\end{figure*}

\section{Introduction}

Scaling laws for large multimodal models (LMMs) show that performance on complex tasks improves steadily as more data is used, as long as the basic quality of the data is maintained \citep{brown2020languagemodelsfewshotlearners, kaplan2020scalinglawsneurallanguage}. This finding has created a strong and growing demand for large-scale training data. However, the collection of high-quality human annotations has not kept pace with the rapid growth of model parameters \cite{wang2024surveydatasynthesisaugmentation, villalobos2024rundatalimitsllm}. As a result, synthetic data has become essential for filling this gap. Synthetic data now serves as a cornerstone across diverse applications, from domain-specific reasoning (e.g., mathematics, coding) to chain-of-thought (CoT) augmentation \cite{lu2025machinelearningsyntheticdata, tong2024dartmathdifficultyawarerejectiontuning, zhou2024jiuzhang30efficientlyimprovingmathematical}, offering unparalleled advantages in scalability and cost-efficiency \cite{chang2024surveydatasynthesisapproaches}.

However, many samples in synthetic datasets are low quality or contain hallucinations, especially when ground truth (GT) is missing or cannot be verified \cite{vanbreugel2023syntheticdatarealerrors, bauer2024comprehensiveexplorationsyntheticdata}. To balance data quality and quantity, the \textit{LLM-as-a-Judge} framework has become a common industry approach. It uses strong models as pseudo-oracles to filter generated data \cite{zheng2023judgingllmasajudgemtbenchchatbot, gu2025surveyllmasajudge}. However, this approach has clear limits. It requires high computational cost due to multi-turn generation, offers little interpretability in its decisions, and risks preference leakage, where the judge model amplifies the biases of the generator \citep{li2025generationjudgmentopportunitieschallenges, li2025preferenceleakagecontaminationproblem}.

Recent work on data efficiency challenges the idea that \textit{more data is always better}. It shows that a small amount of high-quality data can match or even outperform training on full datasets. For example, LIMA \citep{zhou2023limaalignment} shows that model alignment depends more on data quality than on data size. This insight brings renewed attention to data selection, a classic machine learning problem that aims to find a small but highly useful subset of data \cite{hart1968condensed}. However, many traditional selection methods do not scale well to modern LLMs or fail to capture the complexity of reasoning tasks.

To bridge the gap between expensive semantic verification and scalable selection, we propose \textbf{One-Step-Train (OST)}. This framework reformulates selection as an \textit{incremental training utility} ranking problem, estimating the marginal value of each sample via a simulated single-step update on a lightweight proxy. By grounding data quality in the actual optimization signal, OST effectively bypasses the prohibitive costs of retraining-based influence functions while maintaining high alignment with target model performance. Our main contributions are summarized as follows:

\begin{itemize}
    \item We establish OST, an efficient data selection framework that significantly reduces total computational overhead (measured in GPU hours) by 17\%---including selection costs---while achieving a 43\% reduction specifically in the training phase. Simultaneously, it improves model capability, improving upon heuristic scoring baselines like DEITA \citep{liu2024makesgooddataalignment} and surpassing LLM-as-a-Judge baseline by 1.8 points.
    
    \item We find that the correlation between proxy model convergence and selection quality is non-monotonic. Our analysis shows that using a checkpoint with only 5\% warm-up achieves the best trade-off between adapting to the target domain and resisting synthetic noise. In contrast, fully converged proxy models tend to overfit spurious patterns and hallucinations in synthetic data.
    
    \item  Validated on the Qwen3-VL series, OST effectively identifies and purges toxic samples that degrade reasoning logic. We demonstrate a substantial \textit{Less is More} phenomenon where a high-utility Top-20\% subset notably outperforms the uncurated Full-SFT baseline by \textbf{8.8} points, effectively reversing the negative transfer \cite{wang2019characterizingavoidingnegativetransfer} observed in complex benchmarks.
\end{itemize}

\section{Related Work}
\label{sec:Related Work}

We situate our work within two primary research streams: synthetic data evaluation via LLM-as-a-Judge, and efficient data selection grounded in data influence.

\paragraph{Synthetic Data Generation and Verification.}
Synthetic data has emerged as a cornerstone for instruction tuning and reasoning alignment, with recent works demonstrating its efficacy in domain-specific augmentation \citep{lu2025machinelearningsyntheticdata, zhou2024jiuzhang30efficientlyimprovingmathematical, tong2024dartmathdifficultyawarerejectiontuning}. However, to mitigate inherent hallucinations \citep{vanbreugel2023syntheticdatarealerrors} and prevent model collapse \citep{bauer2024comprehensiveexplorationsyntheticdata, shumailov2024curserecursiontraininggenerated}, rigorous quality control is essential. The dominant verification paradigm, \textit{LLM-as-a-Judge}, employs strong models as pseudo-oracles to filter generated samples \citep{zheng2023judgingllmasajudgemtbenchchatbot, gu2025surveyllmasajudge}. While effective, this generative approach is constrained by prohibitive inference costs and the risk of preference leakage \citep{li2025preferenceleakagecontaminationproblem, li2025generationjudgmentopportunitieschallenges}, motivating the shift towards more efficient, optimization-grounded selection metrics.

\paragraph{Data Selection and Influence Functions.}
Data selection aims to identify a training subset (coreset) that approximates or exceeds the performance of the full dataset \citep{mirzasoleiman2020coresetsdataefficienttrainingmachine}. This is often formalized through Influence Functions (IF), which quantify the counterfactual impact of a training point on test loss \citep{koh2020understandingblackboxpredictionsinfluence}. Recent studies affirm that high-quality subsets can drive strong performance in instruction tuning, challenging the necessity of massive datasets \citep{zhou2023limaalignment, chen2023maybe05dataneeded, chen2024alpagasustrainingbetteralpaca}.

Current selection paradigms can be categorized by computational overhead \citep{yin2025computeconstraineddataselection}: 1) \textbf{Heuristic \& Embedding Methods:} Low-cost approaches rely on surface features like perplexity (PPL), response length, or embedding diversity \citep{wettig2024quratingselectinghighqualitydata, zhao2024longalignmentsimpletoughtobeat, liu2024makesgooddataalignment}. While efficient, these metrics are often loosely correlated with downstream reasoning capability. 2) \textbf{Loss-Driven Proxies:} Methods like Montessori-Instruct \citep{li2024montessoriinstructgenerateinfluentialtraining} and S2L \citep{yang2024smalltolarges2lscalabledata} utilize small-model loss trajectories to estimate data utility. Similarly, PreSelect trains lightweight classifiers to predict loss consistency across models \citep{shum2025predictivedataselectiondata}. 3) \textbf{Gradient-Based Methods:} These offer the highest theoretical fidelity by directly linking data to optimization signals. Notable examples include LESS \citep{xia2024lessselectinginfluentialdata}, which projects gradients into a low-rank space for selection, and recent extensions to multimodal settings \citep{liu2024morehighvaluedataselection}.

\section{Theoretical Foundations}
\label{sec:preliminaries}

We formulate synthetic data filtering as a \emph{ranking} problem, assigning a scalar utility to each sample $z_i$ to maximize validation performance. Theoretically, Leave-One-Out (LOO) retraining serves as the ground-truth oracle for measuring exact counterfactual data influence, as it directly quantifies the change in validation loss if a specific sample is removed from the training set. However, calculating exact LOO is computationally prohibitive, requiring $N$ complete retrainings for a dataset of size $N$. Influence Functions (IF) provide an approximation but suffer from Hessian instability in deep non-convex models \citep{basu2021influencefunctionsdeeplearning}. To address these limitations, we introduce \textbf{One-Step-Train (OST)}, an \emph{incremental} proxy that estimates utility by calculating a single step update from a fixed checkpoint $\boldsymbol{\theta}$. As shown in Appendix~\ref{app:appendix_proofs}, this formulation bypasses explicit curvature computation, reducing to a scalable gradient inner-product ranking.

\subsection{LOO Data Influence}
\label{sec:loo_if}

The LOO influence measures the counterfactual validation loss change if a sample $z_i$ is removed from the training set $\mathcal{D}$ \citep{cook1977detection, 10.1093/oso/9780198522669.003.0009}. Since calculating exact LOO requires $|\mathcal{D}|$ retrainings, Influence Functions (IF) approximate this via a second-order Taylor expansion \citep{koh2020understandingblackboxpredictionsinfluence}:
\begin{equation}
I_{\mathrm{IF}}(z_i) = - \nabla L_{\mathcal{V}}(\hat{\boldsymbol{\theta}})^\top \mathbf{H}_{\hat{\boldsymbol{\theta}}}^{-1} \nabla \ell(z_i; \hat{\boldsymbol{\theta}}),
\label{eq:if}
\end{equation}
where $I_{\mathrm{IF}}(z_i)$ represents the estimated influence score of sample $z_i$. In this formulation, $\nabla L_{\mathcal{V}}(\hat{\boldsymbol{\theta}})$ denotes the gradient of the loss on the validation anchor set, representing the direction of optimal model improvement; $\nabla \ell(z_i; \hat{\boldsymbol{\theta}})$ is the training gradient for sample $z_i$, which signifies the specific update direction that the sample provides; and $\mathbf{H}_{\hat{\boldsymbol{\theta}}}^{-1}$ is the inverse Hessian matrix, which adjusts the alignment by accounting for the local curvature of the loss landscape. While IF correlates strongly with LOO in convex settings (Figure~\ref{fig:mnist_sanity_triptych}(a)), the inverse Hessian is prohibitively expensive and numerically unstable for large-scale LLMs, necessitating a Hessian-free approach.

% [Figure 2 placeholder: Sanity Check Linear Regression]
\begin{figure*}[t]
  \centering
  \begin{minipage}[t]{0.317\linewidth}
    \centering
    \includegraphics[width=\linewidth]{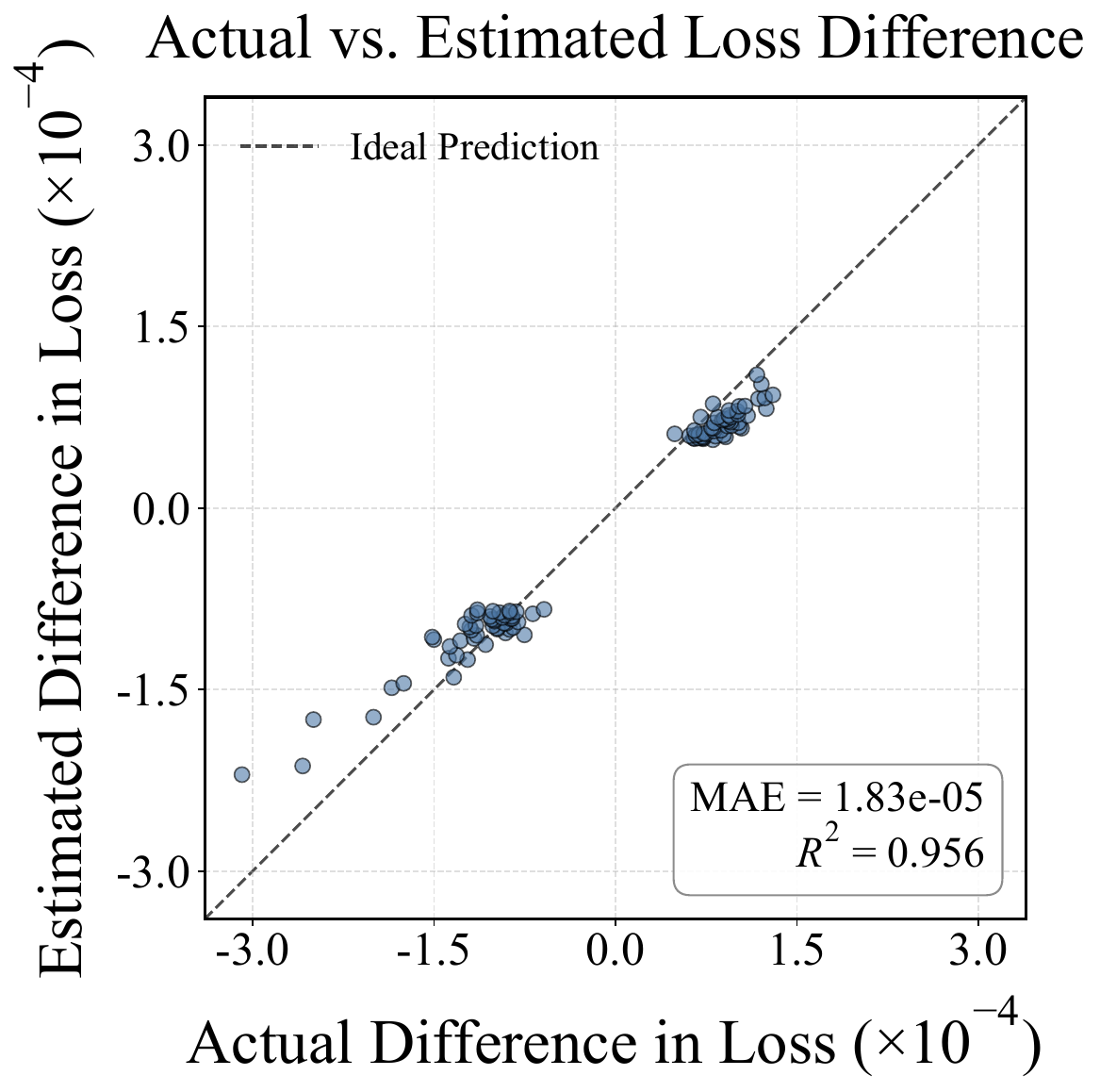}
    \vspace{-0.3em}
    \caption*{(a) IF vs. LOO (scatter)}
  \end{minipage}\hfill
  \begin{minipage}[t]{0.33\linewidth}
    \centering
    \includegraphics[width=\linewidth]{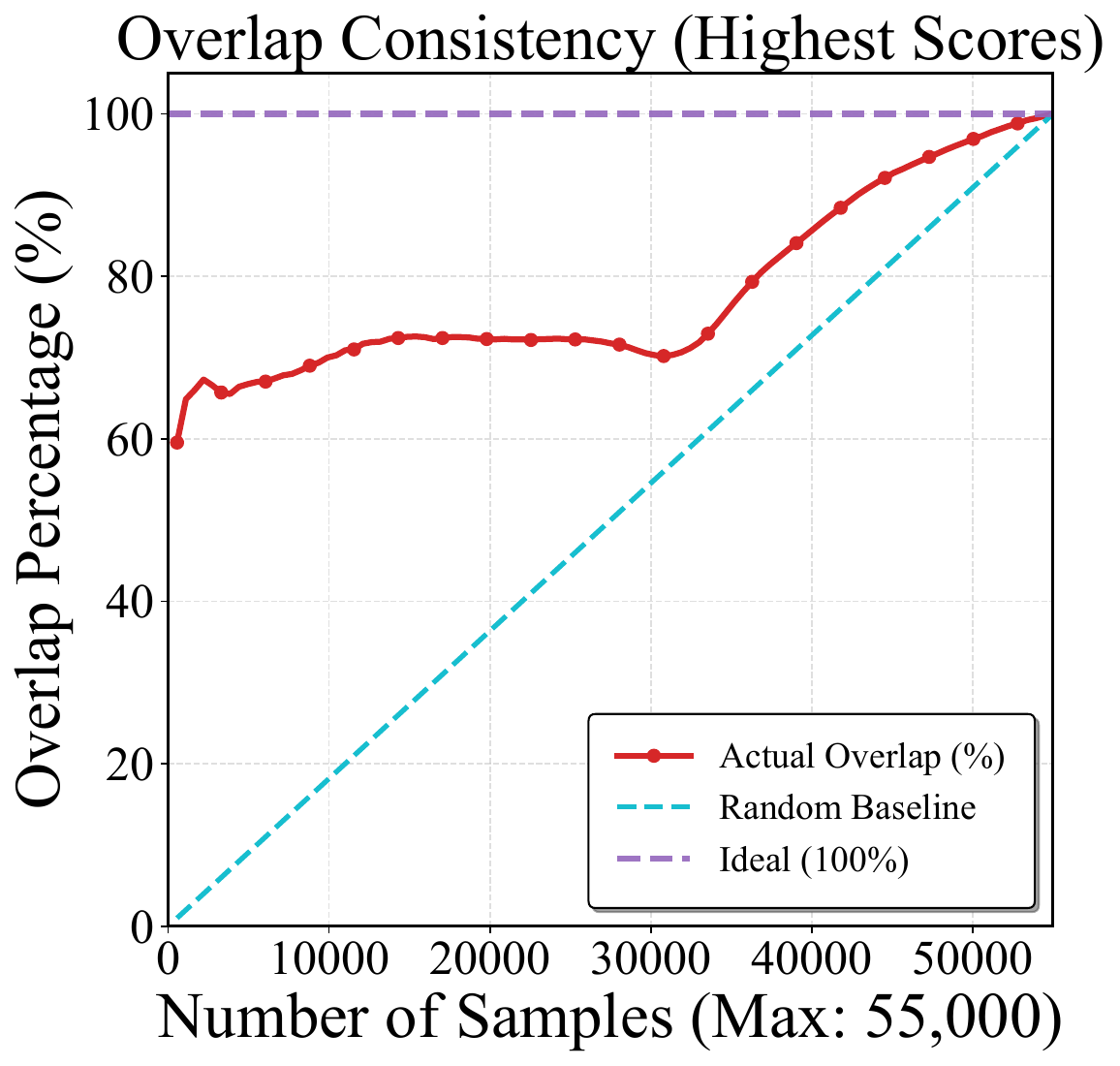}
    \vspace{-0.3em}
    \caption*{(b) Top-$p$\% overlap}
  \end{minipage}\hfill
  \begin{minipage}[t]{0.33\linewidth}
    \centering
    \includegraphics[width=\linewidth]{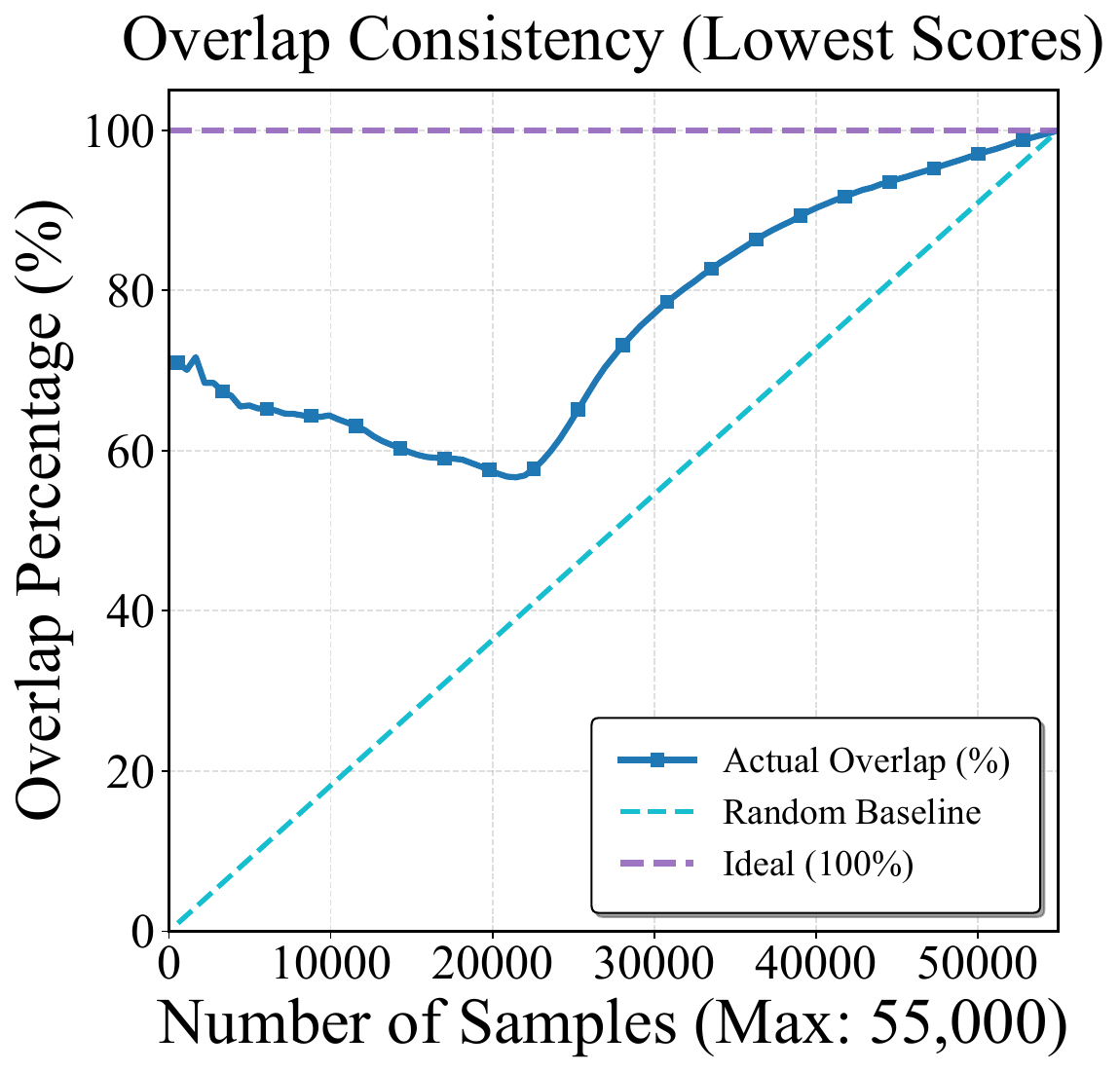}
    \vspace{-0.3em}
    \caption*{(c) Bottom-$p$\% overlap}
  \end{minipage}
    \vspace{-0.3em}
    \caption{Sanity checks on 10-class MNIST (55K examples) using a linear regression model. 
    \textbf{(a) IF vs. LOO:} Comparison between Influence Function (IF) scores and exact Leave-One-Out (LOO) loss changes for the top-200 influential examples. The strong linear alignment validates IF as a reliable proxy for LOO in this convex setting. 
    \textbf{(b)--(c) Ranking Agreement:} Alignment between our \textbf{One-Step-Train} utility and IF. To establish a rigorous evaluation chain, we first validate IF as a computationally tractable proxy for the prohibitively expensive exact LOO in (a), and subsequently utilize IF as the ground truth to benchmark our OST method. The observed overlap for top-$p$\% and bottom-$p$\% subsets significantly exceeds the random baseline (dashed line). Note that as the selected subset size reaches the total dataset volume (55,000), any two ranking methods will naturally select the entire pool, causing the overlap to predictably converge to 100\%.}
    \label{fig:mnist_sanity_triptych}
\end{figure*}

\subsection{Incremental Data Influence}
\label{sec:incremental_influence}

While LOO is fundamentally \textit{retrospective}---measuring how the model degrades if existing data is removed---the training of modern LMMs is inherently \textit{staged} and \textit{incremental}. Following pre-training, models undergo Supervised Fine-Tuning (SFT) to adapt to specific domains or instruction formats \citep{gururangan2020dontstoppretrainingadapt, wei2022finetunedlanguagemodelszeroshot, ouyang2022traininglanguagemodelsfollow}. This paradigm shift necessitates a \textit{prospective} view of data influence: rather than assessing the cost of removal, we focus on the \textbf{marginal utility of a sample in advancing the model from its current state toward the target distribution}. Specifically, we ask: \textit{To what extent does a single optimization step on this sample align the model's trajectory with our validation objective?}

\paragraph{One-Step-Train Utility.}
Starting from a fixed checkpoint $\boldsymbol{\theta}$, we quantify the \textbf{marginal utility} of taking a single optimization step on a specific sample $z_i$. Let $\boldsymbol{g}_{\mathcal{V}}(\boldsymbol{\theta}) \triangleq \nabla L_{\mathcal{V}}(\boldsymbol{\theta})$ denote the gradient on the validation set, and $\boldsymbol{g}_i(\boldsymbol{\theta}) \triangleq \nabla \ell(z_i;\boldsymbol{\theta})$ represent the gradient for sample $z_i$. We simulate a one-step parameter update $\boldsymbol{\theta}_i^{+} = \boldsymbol{\theta} - \eta \boldsymbol{g}_i(\boldsymbol{\theta})$. The resulting \textbf{One-Step-Train (OST) utility}, denoted as $\Delta_i(\boldsymbol{\theta})$, is formally defined as the reduction in validation loss:
\begin{equation}
\Delta_i(\boldsymbol{\theta}) \triangleq L_{\mathcal{V}}(\boldsymbol{\theta}) - L_{\mathcal{V}}(\boldsymbol{\theta}_i^{+}),
\label{eq:delta_def}
\end{equation}
where the symbol $\triangleq$ signifies a formal definition. In this formulation, $\mathcal{V}$ refers to the \textbf{validation anchor set}, a representative subset strictly aligned with the target task distribution to provide the optimization direction. A positive utility score ($\Delta_i > 0$) implies that sample $z_i$ provides a gradient direction beneficial for the target task, directly contributing to model improvement on the anchor set. Conversely, samples with negative utility ($\Delta_i < 0$) are identified as toxic samples that cause negative transfer and must be purged to maintain reasoning logic.

\paragraph{Ranking Consistency (Hessian-Free).}
Although Eq.~\ref{eq:delta_def} is intuitive, it requires expensive forward passes on $\mathcal{V}$ for each candidate. However, assuming the validation loss is $\beta$-smooth, we can apply a first-order Taylor expansion where the remainder is strictly bounded by $\frac{\beta \eta^2}{2} \|\boldsymbol{g}_i(\boldsymbol{\theta})\|^2$. Consequently, the utility is dominated by the linear alignment term:
\begin{equation}
\Delta_i(\boldsymbol{\theta}) \approx \eta \cdot \boldsymbol{g}_{\mathcal{V}}(\boldsymbol{\theta})^\top \boldsymbol{g}_i(\boldsymbol{\theta}).
\end{equation}
This leads to our core ranking theorem:

\begin{theorem}[Ranking Consistency]
\label{thm:ranking_consistency}
Let $s_i(\boldsymbol{\theta}) = \boldsymbol{g}_{\mathcal{V}}(\boldsymbol{\theta})^\top \boldsymbol{g}_i(\boldsymbol{\theta})$ be the gradient inner-product score. If the step size $\eta$ is sufficiently small to satisfy $\eta \le 2\gamma / \beta(\|\boldsymbol{g}_i\|^2 + \|\boldsymbol{g}_j\|^2)$ for a given score margin $\gamma$, the utility ranking perfectly preserves the gradient alignment order:
\begin{equation}
\Delta_i(\boldsymbol{\theta}) > \Delta_j(\boldsymbol{\theta}) \iff s_i(\boldsymbol{\theta}) > s_j(\boldsymbol{\theta}).
\end{equation}
\end{theorem}
We empirically validate this alignment in Figure~\ref{fig:mnist_sanity_triptych}(b)--(c), where OST rankings exhibit significant overlap with the computationally expensive IF rankings, confirming that we can select high-utility data simply by computing dot products, effectively bypassing the Hessian inverse (More details are shown in Appendix~\ref{app:appendix_proofs}).

\subsection{Theoretical Stability and Transferability}
\label{sec:stability}

We briefly address the validity of these instantaneous scores over practical training windows and across model sizes.

\paragraph{Stability over Training.}
Since $\Delta_i$ is local, \textit{does the ranking hold as parameters evolve?} Assuming Lipschitz continuous gradients ($L_{\mathcal{V}}, L_i$) and a bounded parameter drift $\|\boldsymbol{\theta}_t - \boldsymbol{\theta}_0\| \le R$ typical in fine-tuning, the ranking remains mathematically stable (More details are shown in Appendix~\ref{app:proof_theorem_stability}).

\begin{theorem}[Ranking Stability]
If the initial score margin $M_0 = s_{i,0} - s_{j,0}$ dominates the drift bound (i.e., $M_0 > 2[L_{\mathcal{V}}(B_i+B_j)+(L_i+L_j)B_{\mathcal{V}}]R$), the utility ordering at initialization ($t=0$) is strictly preserved at step $t$:
\begin{equation}
s_{i,0} > s_{j,0} \implies s_{i,t} > s_{j,t}.
\end{equation}
\end{theorem}

\paragraph{Cross-Model Alignment.}
To ensure scalability, we estimate utility using a lightweight proxy model. This strategy is theoretically grounded in the \textbf{Subspace Alignment Hypothesis}. Let $\boldsymbol{g}^{(M)}$ and $\boldsymbol{g}^{(m)}$ denote the gradients of the target and proxy models. Assuming they share a low-rank structure differing by an isometry $\mathbf{P}$ and scaling $\alpha$ (i.e., $\|\boldsymbol{g}^{(M)} - \alpha \mathbf{P} \boldsymbol{g}^{(m)}\| \le \epsilon$), the target utility decomposes into $s_i^{(M)} \approx \alpha^2 s_i^{(m)} + \mathcal{R}_{\text{noise}}$. Thus, the ranking order remains consistent across model scales provided the scaled proxy margin exceeds the alignment noise.
\section{Pipeline}
\label{sec:method}

We propose \textbf{One-Step-Train (OST)}, a pipeline that operationalizes the incremental influence framework into a scalable selection process. As detailed in Algorithm \ref{alg:onestep_selection}, OST identifies high-utility subsets through three streamlined phases.

% \subsection{The OST Pipeline}
% \label{sec:method_pipeline}

\paragraph{Phase 1: Utility Scoring.}
For every candidate sample $z_i$, we simulate a single optimization step on the proxy parameters $\theta$. We quantify the sample's value as the resulting reduction in validation loss, denoted as $\Delta_i$ (Eq. \ref{eq:delta_def}). 

\paragraph{Phase 2: Ranking and Filtering.}
We treat data quality as a ranking problem. The synthetic pool is sorted by utility scores $\{\Delta_i\}$, and a hard threshold (Top-$p\%$) is applied to isolate the high-value subset $\mathcal{D}_{\mathrm{sel}}$. This step rigorously filters out toxic samples ($\Delta_i < 0$, causing negative transfer) and ineffective samples ($\Delta_i \approx 0$).

\paragraph{Phase 3: Target Optimization.}
The large-scale target model $\theta_{\mathrm{target}}$ is fine-tuned exclusively on the curated subset $\mathcal{D}_{\mathrm{sel}}$. By training solely on samples with positive validation alignment, we maximize data efficiency while minimizing the risk of model collapse.

\begin{algorithm}[t]
\caption{One-Step-Train Data Selection}
\label{alg:onestep_selection}
\begin{algorithmic}[1]
\Require 
    Synthetic candidate pool $\mathcal{D}_{\mathrm{syn}} = \{z_i\}_{i=1}^N$; 
    Validation set $\mathcal{V}$ (Anchor); 
    Model parameters $\theta$ (or Proxy $\theta_{\mathrm{proxy}}$); 
    Selection ratio $p\%$; 
    Learning rate $\eta$.
\Ensure Selected high-utility subset $\mathcal{D}_{\mathrm{sel}}$.

\Statex \textbf{// Phase 1: Utility Scoring}
\State Compute baseline validation loss:
\State $L_{\text{base}} \leftarrow L_{\mathcal{V}}(\theta)$

\For{each sample $z_i \in \mathcal{D}_{\mathrm{syn}}$}
    \State \textbf{Step 1.1:} Compute gradient on sample $z_i$
    \State $g_i \leftarrow \nabla \ell(z_i; \theta)$ 
    \State \textbf{Step 1.2:} Simulate one-step update
    \State $\theta'_i \leftarrow \theta - \eta \cdot g_i$
    \State \textbf{Step 1.3:} Measure validation impact
    \State $L_{\text{new}} \leftarrow L_{\mathcal{V}}(\theta'_i)$
    \State $\Delta_i \leftarrow L_{\text{base}} - L_{\text{new}}$ 
\EndFor

\Statex \textbf{// Phase 2: Ranking and Selection}
\State Sort $\mathcal{D}_{\mathrm{syn}}$ based on scores $\{\Delta_i\}$ in order.
\State Determine cutoff threshold $K$.
\State $\mathcal{D}_{\mathrm{sel}} \leftarrow \operatorname{Top-K}(\mathcal{D}_{\mathrm{syn}}, \{\Delta_i\})$.

\Statex \textbf{// Phase 3: Downstream Training}
\State Initialize target model $\theta_{\mathrm{target}}$.
\State Train $\theta_{\mathrm{target}}$ on $\mathcal{D}_{\mathrm{sel}}$ (Standard SFT).
\State \Return $\mathcal{D}_{\mathrm{sel}}$
\end{algorithmic}
\end{algorithm}

% \subsection{Experimental Protocols}
% \label{sec:method_proxy}

% \paragraph{Evaluation Protocols.}
% We evaluate $\mathcal{D}_{\mathrm{sel}}$ under two regimes to disentangle data quality from compute budget:
% \begin{itemize}
%     \setlength\itemsep{0em}
%     \item \textbf{Protocol A (Iso-data):} Training steps scale linearly with subset size. This tests \textit{time efficiency} in constrained production environments.
%     \item \textbf{Protocol B (Iso-compute):} Total optimization tokens are fixed (increasing epochs for subsets). This isolates \textit{data density}, testing convergence behavior per unit of compute.
% \end{itemize}

\section{Experiments}

\subsection{Experimental Settings}
\label{sec:exp_setup}

\paragraph{Data Setup.}
We construct a training corpus comprising 351,157 multimodal mathematical problems sourced from real-world examinations. To enhance reasoning density, we augment this dataset by synthesizing Chain-of-Thought (CoT) rationales using \textbf{Doubao-Seed-1.6-thinking}, followed by a rigorous quality verification pipeline. Detailed synthesis protocols, including the inverse rendering technique and filtering hyperparameters, are provided in Appendix~\ref{sec:data_details}.

\paragraph{Anchor Set and Evaluation Benchmarks.}
The OST utility calculation relies on a representative validation anchor $\mathcal{V}$. We employ a stratified subset of 100 examples strictly aligned with the distribution of our held-out test set. For downstream evaluation, we assess performance on:
(1) An \textbf{Internal Benchmark} of 500 questions spanning four difficulty tiers (Statistics in Table~\ref{tab:data_distribution}); and
(2) Four open-source benchmarks: \textbf{MathVision}~\citep{wang2024measuring}, \textbf{MathVista}~\citep{lu2024mathvista}, \textbf{WeMath}~\citep{qiao2024we}, and \textbf{LogicVista}~\citep{xiao2024logicvistamultimodalllmlogical}, to test out-of-distribution generalization.

% \paragraph{Anchor Set and Evaluation Benchmarks.}
% The calculation of the OST utility $\Delta_i$ relies on a representative validation set $\mathcal{V}$ to act as a validation anchor. We compiled a proprietary pool of authentic mathematics examination problems spanning diverse difficulty levels. To ensure rigorous evaluation, this pool is partitioned into distinct subsets:

\begin{table}[h]
\centering
\small
\renewcommand{\arraystretch}{1.2}
\resizebox{1.0\linewidth}{!}{
\begin{tabular}{llc}
\toprule
\textbf{Category} & \textbf{Domain} & \textbf{Size} \\
\midrule
\multirow{3}{*}{\textit{Arithmetic \& Algebra}} & Calculation & 50 \\
 & Number Concepts & 50 \\
 & Word Problems & 50 \\
\midrule 
\multirow{2}{*}{\textit{Geometry \& Visual}} & Geometry & 50 \\
 & Visual \& Graphs & 50 \\
\midrule 
\multirow{4}{*}{\textit{K-12 Benchmarks}} 
 & Middle School Math Exam & 50 \\
 & High School Math Exam & 50 \\
 & Math Olympiad & 50 \\
\midrule 
\multirow{2}{*}{\textit{Higher Education}} & Graduate Exam (Math I) & 50 \\
 & Graduate Exam (Math III) & 50 \\
\midrule
\midrule
\textbf{Total} & \textbf{Internal Test Set} & \textbf{500} \\
\bottomrule
\end{tabular}
}
\vspace{1.0em}
\caption{\textbf{Statistics of the Internal Benchmark.} The dataset is categorized into four capability groups to evaluate performance across varying difficulty levels.}
\label{tab:data_distribution}
\end{table}

% \begin{itemize}
%     \item \textbf{Anchor Set ($\mathcal{V}$):} A subset of 100 examples selected via stratified sampling to strictly follow the same distribution as the internal test set. These samples are exclusively reserved for gradient estimation and are isolated from the training process to prevent data leakage.
    
%     \item \textbf{Internal Benchmark:} A held-out test set of 500 questions designed for fine-grained capability diagnosis across four difficulty tiers (see Table~\ref{tab:data_distribution}).
    
%     \item \textbf{External Benchmarks:} To further assess robust generalization on out-of-distribution tasks, we concurrently evaluate models on four widely recognized open-source benchmarks: MathVision~\citep{wang2024measuring}, MathVista~\citep{lu2024mathvista}, WeMath~\citep{qiao2024we}, and LogicVista~\citep{xiao2024logicvistamultimodalllmlogical}.
% \end{itemize}

% [Insert Table 2 Here]
\begin{table*}[t!]
\centering
\small

\vspace{1.0em}
% ==================== Part A: Overall Budget Impact ====================
\resizebox{0.95\textwidth}{!}{
\begin{tabular}{l|c|ccc|ccc}
\multicolumn{8}{c}{\textbf{(a) Overall Impact of Training Budget \& Epochs}} \\
\toprule
\textbf{Subset} & \textbf{Data Size} & \multicolumn{3}{c|}{\textbf{Protocol A: Proportional Steps}} & \multicolumn{3}{c}{\textbf{Protocol B: Fixed Compute}} \\
\textit{(Selection)} & \textit{(Count)} & \textit{Epochs} & \textit{Steps} & \textit{Avg. Score} & \textit{Epochs} & \textit{Steps} & \textit{Avg. Score} \\
\midrule
Base (Before SFT) & - & - & - & 59.0 & - & - & 59.0 \\
Full-SFT (100\%) & 351.2k & $\approx$1.0 & 2000 & 60.5 \scriptsize{\textcolor{red}{(-3.2)}} & $\approx$1.0 & 2000 & 60.5 \scriptsize{\textcolor{red}{(-3.2)}} \\
\textbf{LLM-as-a-Judge} & 285.2k & $\approx$1.0 & 1500 & \textbf{63.7} & $\approx$1.0 & 1500 & \textbf{63.7} \\
\midrule
Best-10\% & 35.1k & $\approx$1.0 & 200 & 61.7 \scriptsize{\textcolor{red}{(-2.0)}} & $\approx$5.0 & 1000 & 67.5 \scriptsize{\textcolor{teal}{(+3.8)}} \\
\textbf{Best-20\%} & 70.2k & $\approx$1.0 & 400 & 62.9 \scriptsize{\textcolor{red}{(-0.8)}} & $\approx$2.5 & 1000 & \textbf{69.3 \scriptsize{\textcolor{teal}{(+5.6)}}} \\
Best-30\% & 105.3k & $\approx$1.0 & 600 & 62.7 \scriptsize{\textcolor{red}{(-1.0)}} & $\approx$1.6 & 1000 & 64.1 \scriptsize{\textcolor{teal}{(+0.4)}} \\
\textbf{Best-50\%} & 175.6k & $\approx$1.0 & 1000 & \textbf{65.5 \scriptsize{\textcolor{teal}{(+1.8)}}} & $\approx$1.0 & 1000 & 65.5 \scriptsize{\textcolor{teal}{(+1.8)}} \\
\bottomrule
\end{tabular}
}
\vspace{1.0em}
% ==================== Part B: Detailed Breakdown ====================
\resizebox{0.95\textwidth}{!}{
\setlength{\tabcolsep}{2.5pt}
\begin{tabular}{l|cccc|cccc}
\multicolumn{9}{c}{\textbf{(b) Fine-grained Analysis by Category (Actual Score \& $\Delta$)}} \\
\toprule
& \multicolumn{4}{c|}{\textbf{Protocol A: Iso-data}} & \multicolumn{4}{c}{\textbf{Protocol B: Iso-Compute}} \\
\textbf{Subset} & \textbf{Arith. \& Alg.} & \textbf{Geom. \& Vis.} & \textbf{K-12 Bench.} & \textbf{Higher Ed.} & \textbf{Arith. \& Alg.} & \textbf{Geom. \& Vis.} & \textbf{K-12 Bench.} & \textbf{Higher Ed.} \\
\midrule
Base & 70.7 & 59.0 & 44.7 & 29.0 & 70.7 & 59.0 & 44.7 & 29.0 \\
Full-SFT & 74.5 \scriptsize{\textcolor{red}{(-0.8)}} & 57.0 \scriptsize{\textcolor{red}{(-6.0)}} & 60.0 \scriptsize{\textcolor{red}{(-5.3)}} & 28.0 \scriptsize{\textcolor{red}{(-13.0)}} & 74.5 \scriptsize{\textcolor{red}{(-0.8)}} & 57.0 \scriptsize{\textcolor{red}{(-6.0)}} & 60.0 \scriptsize{\textcolor{red}{(-5.3)}} & 28.0 \scriptsize{\textcolor{red}{(-13.0)}} \\
\textbf{Judge} & \textbf{75.3} & \textbf{63.0} & \textbf{65.3} & \textbf{41.0} & \textbf{75.3} & \textbf{63.0} & \textbf{65.3} & \textbf{41.0} \\
\midrule
Best-10\% & 74.0 \scriptsize{\textcolor{red}{(-1.3)}} & 58.0 \scriptsize{\textcolor{red}{(-5.0)}} & 61.3 \scriptsize{\textcolor{red}{(-4.0)}} & \textbf{44.0 \scriptsize{\textcolor{teal}{(+3.0)}}} & 76.6 \scriptsize{\textcolor{teal}{(+1.3)}} & \textbf{65.0 \scriptsize{\textcolor{teal}{(+2.0)}}} & 72.6 \scriptsize{\textcolor{teal}{(+7.3)}} & 45.0 \scriptsize{\textcolor{teal}{(+4.0)}} \\
\textbf{Best-20\%} & 76.0 \scriptsize{\textcolor{teal}{(+0.7)}} & 61.0 \scriptsize{\textcolor{red}{(-2.0)}} & 68.0 \scriptsize{\textcolor{teal}{(+2.7)}} & 34.0 \scriptsize{\textcolor{red}{(-7.0)}} & \textbf{80.0 \scriptsize{\textcolor{teal}{(+4.7)}}} & \textbf{65.0 \scriptsize{\textcolor{teal}{(+2.0)}}} & \textbf{73.3 \scriptsize{\textcolor{teal}{(+8.0)}}} & \textbf{48.0 \scriptsize{\textcolor{teal}{(+7.0)}}} \\
Best-30\% & 74.0 \scriptsize{\textcolor{red}{(-1.3)}} & 58.0 \scriptsize{\textcolor{red}{(-5.0)}} & 70.0 \scriptsize{\textcolor{teal}{(+4.7)}} & 36.0 \scriptsize{\textcolor{red}{(-5.0)}} & 76.6 \scriptsize{\textcolor{teal}{(+1.3)}} & 62.0 \scriptsize{\textcolor{red}{(-1.0)}} & 68.6 \scriptsize{\textcolor{teal}{(+3.3)}} & 37.0 \scriptsize{\textcolor{red}{(-4.0)}} \\
\textbf{Best-50\%} & \textbf{76.6 \scriptsize{\textcolor{teal}{(+1.3)}}} & \textbf{62.0 \scriptsize{\textcolor{red}{(-1.0)}}} & \textbf{72.0 \scriptsize{\textcolor{teal}{(+6.7)}}} & 39.0 \scriptsize{\textcolor{red}{(-2.0)}} & 76.6 \scriptsize{\textcolor{teal}{(+1.3)}} & 62.0 \scriptsize{\textcolor{red}{(-1.0)}} & 72.0 \scriptsize{\textcolor{teal}{(+6.7)}} & 39.0 \scriptsize{\textcolor{red}{(-2.0)}} \\
\bottomrule
\end{tabular}
}
\caption{\textbf{Impact of Data Quality vs. Compute Budget.} \textbf{(a)} Under Protocol B (Fixed Compute), the \textbf{Best-20\%} subset achieves the highest Average Score (\textcolor{teal}{+5.6\%}), significantly outperforming both the \textit{LLM-as-a-Judge} baseline and the uncurated Full-SFT setting. This validates the \textit{Less is More} hypothesis, where prioritizing data density yields superior convergence over mere data quantity. \textbf{(b)} The fine-grained breakdown reveals that while Full-SFT struggles with complex reasoning (e.g., Higher Ed), the OST-selected subsets effectively unlock these capabilities, achieving substantial gains over the Base model.}
\label{tab:combined_ablation}
\end{table*}

% ==================== MODIFIED EFFICIENCY TABLE ====================
\begin{table*}[t!]
\centering
\large
\renewcommand{\arraystretch}{1.25}
\setlength{\tabcolsep}{4pt}
\begin{tabular}{l|cc|cccc|c}
\toprule
\multirow{2}{*}{\textbf{Method}} & \multicolumn{2}{c|}{\textbf{Phase 1: Selection}} & \multicolumn{4}{c|}{\textbf{Phase 2: Training (SFT)}} & \multirow{2}{*}{\textbf{Avg. Score}} \\
\cline{2-7}
& \textbf{Resource} & \textbf{Time} & \textbf{Resource} & \textbf{Data Size} & \textbf{Epochs} & \textbf{Time} & \\
\midrule
% Split LLM-as-a-Judge row for vertical centering
\multirow{2}{*}{LLM-as-a-Judge} & API, 0 GPU & \multirow{2}{*}{5.0h$^\dagger$} & \multirow{2}{*}{64 $\times$ H100} & \multirow{2}{*}{285.2k} & \multirow{2}{*}{1} & \multirow{2}{*}{4.2h} & \multirow{2}{*}{63.7 \scriptsize{\textit{(Base)}}} \\
 & (3 Rounds) & & & & & & \\
\midrule
\textbf{OST (Top-50\%)} & \multirow{4}{*}{\shortstack{16 $\times$ H100\\(Block-wise)}} & \multirow{4}{*}{4.4h$^\dagger$} & \multirow{4}{*}{64 $\times$ H100} & 175.6k & 1 & 2.4h & \textbf{65.5 \inc{1.8}} \\
OST (Top-30\%) & & & & 105.3k & 1 & 1.5h & 62.7 \dec{1.0} \\
OST (Top-20\%) & & & & 70.2k & 1 & 0.9h & 62.9 \dec{0.8} \\
OST (Top-10\%) & & & & 35.1k & 1 & 0.5h & 61.7 \dec{2.0} \\
\bottomrule
\end{tabular}
\vspace{1.0em}
\caption{\textbf{Efficiency and Performance Trade-off on Internal Benchmark.} This table compares the wall-clock time and resulting performance. In terms of total computational cost, the Judge baseline consumes approximately \textbf{269 GPU Hours} ($64 \text{ GPUs} \times 4.2\text{h}$) for training. In contrast, even when including the selection overhead ($16 \text{ GPUs} \times 4.4\text{h} \approx 70$ GPU Hours), the \textbf{Top-50\%} OST pipeline consumes only \textbf{224 GPU Hours} in total ($70 + 154$). This achieves a \textbf{17\% reduction} in total compute resources (and a \textbf{43\% reduction} in the expensive training phase) while boosting accuracy by 1.8 points. ($^\dagger$OST or LLM-as-a-Judge selection is a one-time cost).}
\label{tab:efficiency_breakdown}
\end{table*}

\paragraph{Baselines.}
We evaluate OST against three comparative strategies to isolate the impact of data quality:
(1) \textbf{Full SFT}: Fine-tuning on the complete synthetic corpus (351k samples) to benchmark the performance ceiling of data quantity versus the risk of noise accumulation;
(2) \textbf{Random Selection}: Uniformly sampling subsets to establish a lower bound, representing the expected performance without active filtering;
(3) \textbf{LLM-as-a-Judge}: A strong baseline representing the industry standard, which utilizes Qwen3VL-235B-A22B-Instruct \cite{bai2025qwen3vltechnicalreport} to verify reasoning correctness for rejection sampling (see Appendix~\ref{app:prompts} for the verification prompt).
(4) \textbf{DEITA}: A state-of-the-art heuristic scoring baseline for data selection \cite{liu2024makesgooddataalignment}. To adapt this method to our multimodal reasoning context, we employ its open-source scorers to evaluate the textual complexity and quality of the generated CoT rationales. The candidate samples are then ranked by their combined score ($complexity \times quality$) and iteratively selected using its embedding-based diversity module (Repr Filter) to match our target data budget.

\paragraph{Implementation Framework.}
We adopt a decoupled proxy-target framework to balance selection cost and training performance. For utility scoring (Phase 1), we employ InternVL3-1B ~\citep{zhu2025internvl3exploringadvancedtraining} as a lightweight computational proxy. For downstream validation (Phase 3), we utilize two distinct target architectures: a proprietary 30B Multimodal Model (integrating a Qwen3-30B-A3B language backbone \cite{yang2025qwen3technicalreport} with a ViT-300M vision encoder \cite{dosovitskiy2021imageworth16x16words}) and the open-source Qwen3-VL series (2B, 4B, and 8B). Specific training hyperparameters and hardware configurations and detailed setups for the baselines (including \textit{LLM-as-a-Judge} and \textit{DEITA}) are provided in Appendix~\ref{app:experimental_implementation}.

\paragraph{Evaluation Protocols.}
To rigorously disentangle the impact of data quality from the computational budget, we evaluate the selected subsets $\mathcal{D}_{\mathrm{sel}}$ under two distinct training regimes:
\begin{itemize}
    \setlength\itemsep{0em}
    \item \textbf{Protocol A (Proportional Steps):} The number of training epochs is fixed across all subsets, meaning the total training steps scale linearly with the subset size. This setting evaluates \textit{time efficiency}, demonstrating how much training time can be saved in constrained production environments by using a smaller, curated dataset.
    \item \textbf{Protocol B (Fixed Compute):} The total computational budget (i.e., total optimization steps) is strictly fixed across all subsets, which naturally requires more training epochs for smaller subsets. This setting isolates intrinsic \textit{data density}, testing whether training longer on a high-quality subset yields better convergence per unit of compute without severe overfitting.
\end{itemize}

% [Insert Table 3 Here]
% [Insert Table 3 Here]
\begin{table*}[t!]
\centering
\small
\resizebox{0.87\textwidth}{!}{
\begin{tabular}{l|cccc|cc}
\toprule
\textbf{Model} & \textbf{MathVista} & \textbf{MathVision} & \textbf{WeMath} & \textbf{LogicVista} & \textbf{Avg.} & \textbf{$\Delta$} \\
\scriptsize{\textit{(Total Samples)}} & \scriptsize{\textit{(1000)}} & \scriptsize{\textit{(3040)}} & \scriptsize{\textit{(525)}} & \scriptsize{\textit{(447)}} & & \\
\midrule
\multicolumn{7}{l}{\textit{\textbf{Closed-Source SOTA Models}}} \\
Gemini3-pro & 88.5 & 83.4 & 73.8 & 81.4 & 81.8 & - \\
GPT-5-20250807 & 81.9 & 72.0 & 71.0 & 70.0 & 73.7 & - \\
Doubao-Seed-1.6 & 85.9 & 67.8 & 78.5 & 72.5 & 76.2 & - \\
GLM-4.5V & 84.6 & 65.6 & 68.8 & 62.4 & 70.4 & - \\
\midrule
\multicolumn{7}{l}{\textit{\textbf{Open-Source Large Baselines}}} \\
Qwen3-VL-235B-Thinking & 85.9 & 74.6 & 74.9 & 72.3 & 76.9 & - \\
Kimi-vl-A3B-thinking & 80.1 & 56.8 & 47.0 & 51.0 & 58.7 & - \\
\midrule
\midrule
\multicolumn{7}{l}{\textit{\textbf{Data Selection Strategy}}} \\

% ================= 2B Group =================
Qwen3-VL-2B-Instruct (Base) & 64.1 & 34.5 & 35.8 & 46.8 & 45.3 & - \\
\quad + Full SFT (100\%) & 63.4 \dec{0.7} & 34.1 \dec{0.4} & 34.9 \dec{0.9} & 45.6 \dec{1.2} & 44.5 & \textcolor{red}{-0.8} \\
\quad + Random (20\%) & 63.8 \dec{0.3} & 34.3 \dec{0.2} & 35.4 \dec{0.4} & 45.9 \dec{0.9} & 44.9 & \textcolor{red}{-0.4} \\
\quad + LLM-as-a-Judge (20\%) & 64.5 \inc{0.4} & \textbf{35.0} \inc{0.5} & 36.2 \inc{0.4} & 46.5 \dec{0.3} & 45.5 & \textcolor{teal}{+0.3} \\
\quad + DEITA (Top-20\%) & 64.6 \inc{0.5} & 34.9 \inc{0.4} & 36.5 \inc{0.7} & 46.7 \dec{0.1} & 45.7 & \textcolor{teal}{+0.4} \\
\quad + \textbf{OST (Ours, Best-20\%)}  & \textbf{64.9} \inc{0.8} & 34.8 \inc{0.3} & \textbf{37.5} \inc{1.7} & \textbf{47.0} \inc{0.2} & \textbf{46.1} & \textcolor{teal}{+0.8} \\
\midrule

% ================= 4B Group =================
Qwen3-VL-4B-Instruct (Base) & 74.9 & 54.1 & 54.3 & 60.0 & 60.8 & - \\
\quad + Full SFT (100\%) & 73.2 \dec{1.7} & 53.8 \dec{0.3} & 52.6 \dec{1.7} & 58.2 \dec{1.8} & 59.5 & \textcolor{red}{-1.3} \\
\quad + Random (20\%) & 74.2 \dec{0.7} & 54.0 \dec{0.1} & 53.9 \dec{0.4} & 59.5 \dec{0.5} & 60.4 & \textcolor{red}{-0.4} \\
\quad + LLM-as-a-Judge (20\%) & 75.3 \inc{0.4} & 54.6 \inc{0.5} & 55.0 \inc{0.7} & 59.7 \dec{0.3} & 61.2 & \textcolor{teal}{+0.4} \\
\quad + DEITA (Top-20\%) & 75.1 \inc{0.2} & 54.5 \inc{0.4} & 55.2 \inc{0.9} & 59.6 \dec{0.4} & 61.1 & \textcolor{teal}{+0.3} \\
\quad + \textbf{OST (Ours, Best-20\%)}  & \textbf{75.8} \inc{0.9} & \textbf{54.8} \inc{0.7} & \textbf{56.0} \inc{1.7} & \textbf{60.4} \inc{0.4} & \textbf{61.8} & \textcolor{teal}{+1.0} \\
\midrule

% ================= 8B Group =================
Qwen3-VL-8B-Instruct (Base) & 77.2 & 57.8 & 57.5 & 61.3 & 63.5 & - \\
\quad + Full SFT (100\%) & 70.5 \dec{6.7} & 54.8 \dec{3.0} & 55.2 \dec{2.3} & 60.0 \dec{1.3} & 60.1 & \textcolor{red}{-3.4} \\
\quad + Random (20\%) & 76.1 \dec{1.1} & 57.1 \dec{0.7} & 56.8 \dec{0.7} & 60.6 \dec{0.7} & 62.7 & \textcolor{red}{-0.8} \\
\quad + LLM-as-a-Judge (20\%) & \textbf{77.6} \inc{0.4} & 58.2 \inc{0.4} & 58.1 \inc{0.6} & 60.9 \dec{0.4} & 63.7 & \textcolor{teal}{+0.2} \\
\quad + DEITA (Top-20\%) & 77.5 \inc{0.3} & 58.3 \inc{0.5} & 58.6 \inc{1.1} & 61.2 \dec{0.1} & 63.9 & \textcolor{teal}{+0.4} \\
\quad + \textbf{OST (Ours, Best-20\%)} & 77.4 \inc{0.2} & \textbf{58.5} \inc{0.7} & \textbf{59.0} \inc{1.5} & \textbf{61.5} \inc{0.2} & \textbf{64.1} & \textcolor{teal}{+0.6} \\
\bottomrule
\end{tabular}
}

\vspace{1.0em}
\caption{\textbf{Benchmarking Data Selection Strategies.} Comparative results on Qwen3-VL series. The table highlights the trade-offs between different selection methods. While the heuristic scoring baseline (\textit{DEITA \citep{liu2024makesgooddataalignment} Top-20\%}) performs competitively and often matches the \textit{LLM-as-a-Judge} approach, both methods fall short on complex logic tasks compared to optimization-grounded methods. Specifically, they struggle with deep reasoning logic (\textit{LogicVista}), sometimes degrading performance compared to Base. In contrast, \textbf{OST (Best-20\%)} consistently improves complex reasoning capabilities (\textit{WeMath}, \textit{LogicVista}) and provides the highest average gain across all model scales, effectively countering the negative transfer observed in Full SFT.}
\label{tab:selection_comparison}
\end{table*}

% ==================== FIGURE (Double Column) ====================
\begin{figure*}[t]
    \centering
    \includegraphics[width=0.95\linewidth]{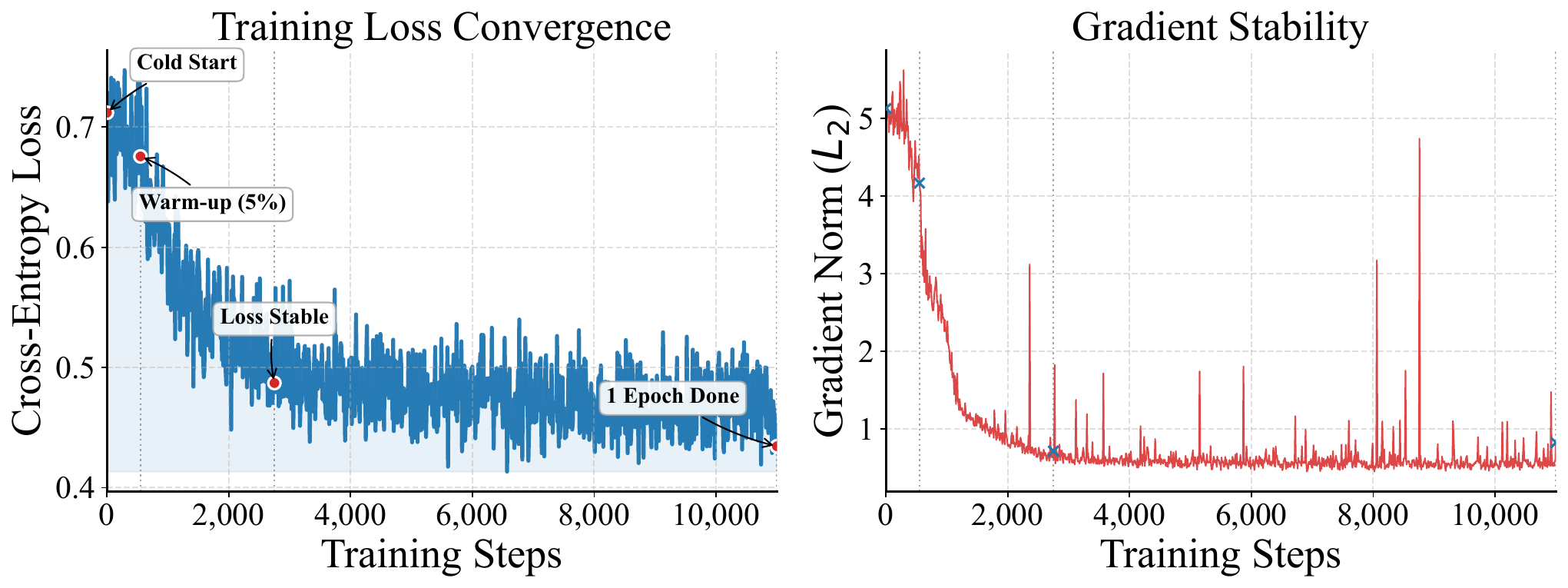}
    \vspace{-0.5em}
    \caption{\textbf{Training Dynamics and Checkpoint Selection.} Checkpoints are saved at four representative stages to score the held-out data pool.}
    \label{fig:training_dynamics}
\end{figure*}
% ================================================================

\subsection{Main Results}
\label{sec:results}

We empirically validate the One-Step-Train (OST) framework across computational efficiency and downstream model capability.

\paragraph{Utility Distribution and Qualitative Analysis.}
To elucidate the selection mechanism, we visualize the distribution of proxy utility scores in Figure~\ref{fig:score_distribution}. The distribution exhibits a long left tail (negative values), corresponding to samples that effectively reduce validation loss ($\Delta_i > 0$ in our definition). Conversely, the probability mass centered around zero or positive values represents samples with negligible or conflicting gradient directions. By applying a hard threshold, OST explicitly isolates the high-utility region while filtering out the toxic tail that potentially destabilizes training. For a concrete analysis of the semantic differences between these partitions, we refer readers to the Qualitative Case Study in Appendix~\ref{app:Case_Study}. Additionally, we verify the statistical independence of these utility scores in Appendix~\ref{app:experimental_implementation}, confirming that the ranking is robust to processing order.

\begin{figure}[H]
    \centering
    \includegraphics[width=0.99\linewidth]{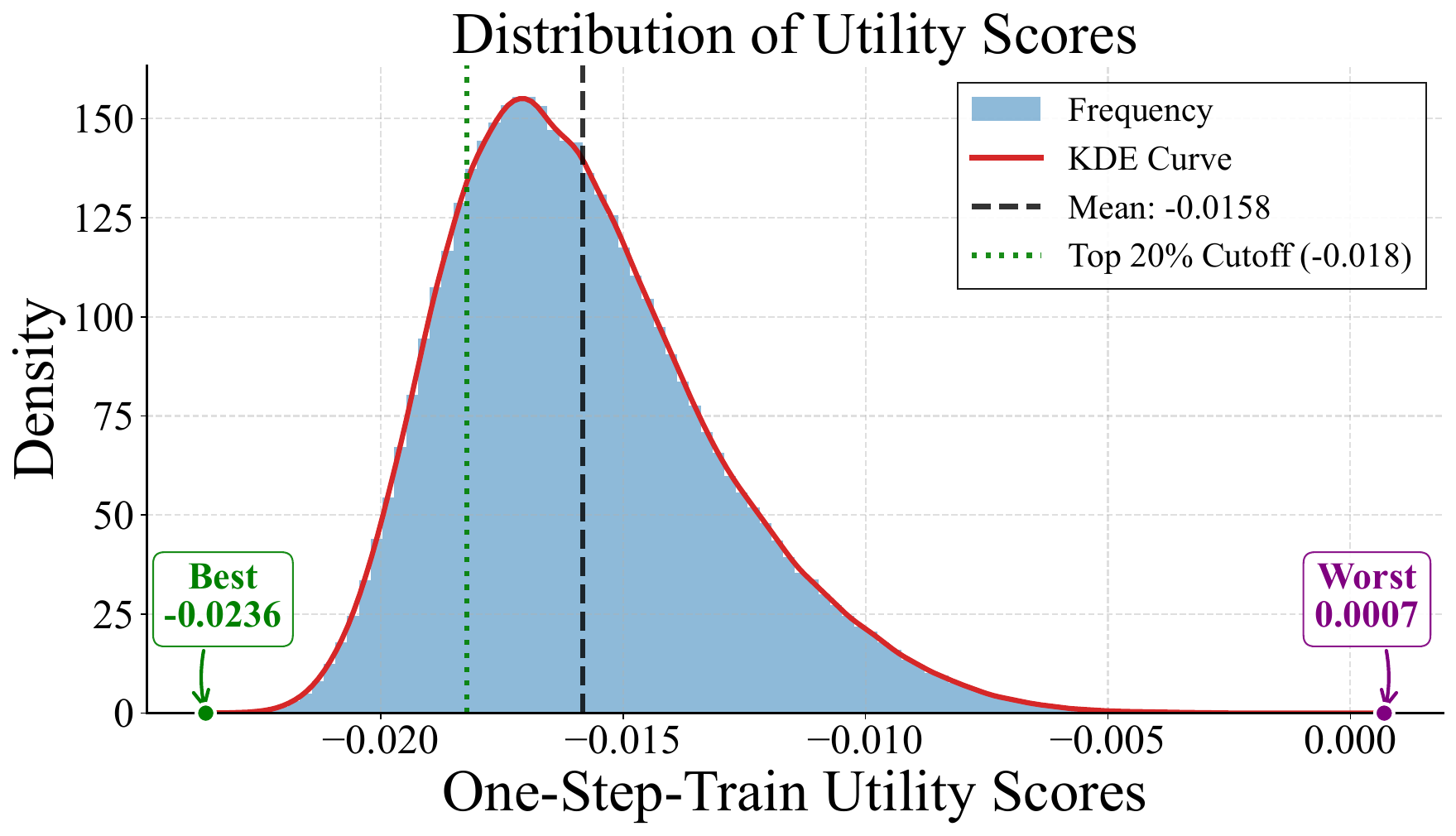} 
    \caption{\textbf{Distribution of One-Step-Train Utility Scores.} Samples in the Best region (left tail) exhibit strong alignment with the validation anchor, whereas Worst samples (right tail) introduce noise or conflicting updates. The vertical dotted line marks the Top-20\% selection threshold.}
    \label{fig:score_distribution}
\end{figure}

\paragraph{Data Density.}
We further isolate the impact of data quality under a fixed compute budget (Protocol B in Table~\ref{tab:combined_ablation}). A \textit{Less is More} phenomenon is observed: the \textbf{Best-20\%} subset yields the highest return on compute, achieving an average score of \textbf{69.3}. This performance significantly outperforms the industry-standard \textit{LLM-as-a-Judge} baseline (63.7) and the larger Top-50\% subset (65.5). Most notably, it surpasses the Full-SFT setting (60.5) by a substantial margin of \textbf{8.8} points. These results suggest that deep optimization (more epochs) on a compact, high-gradient-utility core is far more effective than shallow training on large, noisy corpora.

\paragraph{Pareto-Optimal Cost.}
Table~\ref{tab:efficiency_breakdown} delineates the prohibitive computational overhead inherent in the semantic filtering paradigm: the \textit{LLM-as-a-Judge} baseline incurs substantial training overhead (269 GPU Hours) due to high data retention. In contrast, OST achieves a \textbf{Pareto improvement} in resource allocation. Even accounting for the selection overhead (approx. 70 GPU Hours), the Top-50\% pipeline reduces the total computational cost by \textbf{17\%}. Notably, for the repetitive SFT phase, OST reduces resource consumption by \textbf{43\%} (154h vs. 269h) while simultaneously improving average accuracy by 1.8 points, demonstrating superior scalability for production environments.

\paragraph{Generalization and Negative Transfer.}
Finally, we verify the transferability of our findings to the Qwen3-VL series (Table~\ref{tab:selection_comparison}). A critical finding is the reversal of \textit{negative transfer} in complex reasoning tasks. On benchmarks like \textit{LogicVista}, Full SFT degrades performance compared to the base model (e.g., -1.8 points on Qwen3-VL-4B), likely due to overfitting on hallucinated reasoning chains. While heuristic scoring methods like \textit{DEITA}\citep{liu2024makesgooddataalignment} successfully mitigate this severe degradation by filtering out obvious low-quality noise, they still fall short of optimization-grounded methods in maximizing reasoning gains. Because surface-level heuristics (e.g., length, structural complexity) cannot reliably detect whether a seemingly "complex" reasoning chain contains subtle mathematical flaws, they may still retain samples that limit the model's reasoning potential. In contrast, OST identifies and purges toxic samples based on their actual optimization utility, achieving the highest net positive gain (+0.4). This confirms that gradient-based utility serves as a more robust, model-agnostic metric for filtering subtle logical errors than heuristic and semantic judges.

\subsection{Ablation: Utility Dynamics across Training States}
\label{sec:stability_ablation}

To investigate the impact of proxy optimization on selection quality, we conducted a split-half ablation study. We fine-tuned the InternVL3-1B proxy on the synthetic set $\mathcal{D}_{\text{train}}$ and evaluated checkpoints from four phases (Figure~\ref{fig:training_dynamics}) by filtering the held-out $\mathcal{D}_{\text{pool}}$.

% ==================== TABLE (Single Column) ====================
\begin{table}[h]
\centering
\small
\renewcommand{\arraystretch}{1.25}
\setlength{\tabcolsep}{10pt}
\resizebox{1.0\linewidth}{!}{
\begin{tabular}{l|c|cc}
\toprule
\textbf{Proxy Checkpoint} & \textbf{Progress} & \textbf{Avg. Score} & \textbf{$\Delta$} \\
\midrule
\textit{baseline} & - & 63.5 & - \\
\midrule
Cold Start & 0\% & 63.9 & \textcolor{teal}{+0.4} \\
\textbf{Warm-up} & \textbf{5\%} & \textbf{64.1} & \textcolor{teal}{\textbf{+0.6}} \\
Loss Stable & 25\% & 63.9 & \textcolor{teal}{+0.4} \\
Converged & 100\% & 63.7 & \textcolor{teal}{+0.2} \\
\bottomrule
\end{tabular}
}
\caption{\textbf{Ablation on Proxy Training State.} The Warm-up checkpoint achieves optimal selection performance. Further training degrades discriminatory capability due to overfitting to synthetic noise.}
\label{tab:checkpoint_ablation}
\end{table}
% ===============================================================

Table~\ref{tab:checkpoint_ablation} reveals a non-monotonic trend where performance peaks at the \textbf{Warm-up} stage (5\%, \textbf{+0.6}). While the pre-trained proxy (0\%) offers effective zero-shot filtering (+0.4), further training to convergence degrades performance (+0.2). We attribute this drop to \textbf{noise overfitting} \cite{arpit2017closerlookmemorizationdeep}: a converged proxy assimilates synthetic hallucinations, losing its discriminative power. Thus, a minimally tuned proxy optimally balances CoT adaptation with noise robustness.

\section{Conclusion}

This work reformulates synthetic data curation by shifting from opaque semantic filtering to transparent, optimization-grounded selection. Our \textbf{One-Step-Train (OST)} framework demonstrates that marginal data utility can be effectively captured via simulated gradient updates on lightweight proxies, bypassing the overhead of traditional influence functions. The empirical gains achieved across the QwenVL underscore a critical insight: model scaling is constrained less by data scarcity than by the accumulation of reasoning hallucinations. By identifying and purging these toxic samples, OST provides a Pareto-optimal path for data-efficient training. Looking forward, grounding data quality in optimization dynamics offers a robust foundation for autonomous, self-improving reasoning systems.

\section*{Limitations}

\paragraph{Dependency on Anchor Quality.}
OST relies on the validation set $\mathcal{V}$ to define the optimization direction. Consequently, the method is sensitive to the representativeness of $\mathcal{V}$. If the anchor set is biased or insufficient, the utility metric $\Delta_i$ risks guiding the model to overfit specific artifacts rather than acquiring generalized reasoning, posing a challenge for subjective domains where ground truth is ambiguous.

\paragraph{Point-wise Ranking vs. Collective Influence.}
OST treats selection as a greedy, point-wise ranking problem. As noted in \textit{Most Influential Subset Selection (MISS)} research \citep{hu2025influentialsubsetselectionchallenges}, this additive approach neglects data redundancy and interaction effects. Consequently, OST may select a subset that is individually high-scoring but collectively suboptimal (e.g., lacking diversity) compared to holistic subset selection strategies.

\paragraph{Bounds of Proxy Transferability.}
While our experiments validate utility transfer from 1B to 8B models (supported by the Subspace Alignment theory in Appendix~\ref{app:proof_theorem_projection}), the limits of this mechanism remain unexplored. The stability of gradient alignment across heterogeneous architectures (e.g., Transformer to SSM) or extreme parameter scales warrants further theoretical investigation to bound the projection error.

\bibliographystyle{ACM-Reference-Format}
\bibliography{sample-sigconf-authordraft}

\clearpage
\appendix

\section{Theoretical Proofs}
\label{app:appendix_proofs}

In this section, we provide detailed derivations and proofs for the theoretical results presented in Section \ref{sec:preliminaries}. We begin by formally stating the smoothness and boundedness assumptions that underpin our \textbf{One-Step-Train (OST)} framework.

\paragraph{General Assumptions.}
Throughout the proofs, we assume the validation loss function $L_{\mathcal{V}}(\boldsymbol{\theta})$ and sample-wise loss $\ell(z_i; \boldsymbol{\theta})$ satisfy the following properties within the local optimization neighborhood:
\begin{itemize}
    \item \textbf{$\beta$-Smoothness:} The gradient $\nabla L_{\mathcal{V}}$ is $\beta$-Lipschitz continuous, i.e., $\|\nabla L_{\mathcal{V}}(\boldsymbol{\theta}') - \nabla L_{\mathcal{V}}(\boldsymbol{\theta})\| \le \beta \|\boldsymbol{\theta}' - \boldsymbol{\theta}\|$.
    \item \textbf{Bounded Gradients:} Gradient norms are bounded by constants $B_{\mathcal{V}}$ and $B_i$, such that $\|\nabla L_{\mathcal{V}}\| \le B_{\mathcal{V}}$ and $\|\nabla \ell(z_i)\| \le B_i$.
\end{itemize}

\subsection{Proof of One-Step-Train Utility Expansion}
\label{app:proof_lemma_taylor}

We first prove that the utility $\Delta_i(\boldsymbol{\theta})$ is dominated by the gradient inner product.

\textbf{Lemma 1.} \textit{If $L_{\mathcal{V}}$ is $\beta$-smooth, then the utility can be expanded as $\Delta_i(\boldsymbol{\theta}) = \eta \, \boldsymbol{g}_{\mathcal{V}}(\boldsymbol{\theta})^\top \boldsymbol{g}_i(\boldsymbol{\theta}) + r_i$, where the remainder is bounded by $|r_i| \le \frac{\beta \eta^2}{2} \|\boldsymbol{g}_i(\boldsymbol{\theta})\|^2$.}

\begin{proof}
Recall the one-step update rule: $\boldsymbol{\theta}_i^{+} = \boldsymbol{\theta} - \eta \boldsymbol{g}_i(\boldsymbol{\theta})$. Let the parameter update vector be $\boldsymbol{d}_i = \boldsymbol{theta}_i^{+} - \boldsymbol{\theta} = - \eta \boldsymbol{g}_i(\boldsymbol{\theta})$.
By Taylor's theorem with the Lagrange remainder form, for the function $L_{\mathcal{V}}(\boldsymbol{\theta})$, there exists a point $\boldsymbol{\xi}$ on the line segment connecting $\boldsymbol{\theta}$ and $\boldsymbol{\theta}_i^{+}$ such that:
\begin{align}
L_{\mathcal{V}}(\boldsymbol{\theta}_i^{+}) &= L_{\mathcal{V}}(\boldsymbol{\theta}) + \nabla L_{\mathcal{V}}(\boldsymbol{\theta})^\top \boldsymbol{d}_i \nonumber \\
&\quad + \frac{1}{2} \boldsymbol{d}_i^\top \nabla^2 L_{\mathcal{V}}(\boldsymbol{\xi}) \boldsymbol{d}_i.
\end{align}
Rearranging terms to match the definition of utility $\Delta_i(\boldsymbol{\theta}) \triangleq L_{\mathcal{V}}(\boldsymbol{\theta}) - L_{\mathcal{V}}(\boldsymbol{\theta}_i^{+})$:
\begin{align}
\Delta_i(\boldsymbol{\theta}) &= - \nabla L_{\mathcal{V}}(\boldsymbol{\theta})^\top \boldsymbol{d}_i - \frac{1}{2} \boldsymbol{d}_i^\top \nabla^2 L_{\mathcal{V}}(\boldsymbol{\xi}) \boldsymbol{d}_i \nonumber \\
&= - \boldsymbol{g}_{\mathcal{V}}(\boldsymbol{\theta})^\top (-\eta \boldsymbol{g}_i(\boldsymbol{\theta})) - \frac{1}{2} \boldsymbol{d}_i^\top \nabla^2 L_{\mathcal{V}}(\boldsymbol{\xi}) \boldsymbol{d}_i \nonumber \\
&= \eta \, \boldsymbol{g}_{\mathcal{V}}(\boldsymbol{\theta})^\top \boldsymbol{g}_i(\boldsymbol{\theta}) + r_i,
\end{align}
where the remainder term is $r_i = - \frac{1}{2} \boldsymbol{d}_i^\top \nabla^2 L_{\mathcal{V}}(\boldsymbol{\xi}) \boldsymbol{d}_i$.

Under the $\beta$-smoothness assumption, the spectral norm (largest eigenvalue) of the Hessian is bounded by $\beta$, i.e., $\|\nabla^2 L_{\mathcal{V}}(\boldsymbol{\xi})\|_2 \le \beta$. We can thus bound the magnitude of the remainder:
\begin{align}
|r_i| &\le \frac{1}{2} \|\boldsymbol{d}_i\|^2 \|\nabla^2 L_{\mathcal{V}}(\boldsymbol{\xi})\|_2 \nonumber \\
&\le \frac{\beta}{2} \| - \eta \boldsymbol{g}_i(\boldsymbol{\theta}) \|^2 = \frac{\beta \eta^2}{2} \|\boldsymbol{g}_i(\boldsymbol{\theta})\|^2.
\end{align}
This confirms that for small $\eta$, the linear term dominates, validating Eq. (4) in the main text.
\end{proof}

\subsection{Proof of Ranking Consistency}
\label{app:proof_theorem_consistency}

\textbf{Theorem 1.} \textit{Let $s_i(\boldsymbol{\theta})= \boldsymbol{g}_{\mathcal{V}}(\boldsymbol{\theta})^\top \boldsymbol{g}_i(\boldsymbol{\theta})$. If the score margin satisfies $|s_i(\boldsymbol{\theta}) - s_j(\boldsymbol{\theta})| \ge \gamma$ and $\eta$ is sufficiently small, then $\Delta_i(\boldsymbol{\theta}) > \Delta_j(\boldsymbol{\theta}) \iff s_i(\boldsymbol{\theta}) > s_j(\boldsymbol{\theta})$.}

\begin{proof}
Using the expansion from Lemma 1, the utility difference between two samples $z_i$ and $z_j$ is:
\begin{align}
\Delta_i - \Delta_j &= (\eta s_i + r_i) - (\eta s_j + r_j) \nonumber \\
&= \eta (s_i - s_j) + (r_i - r_j).
\end{align}
Without loss of generality, assume $s_i > s_j$, implying $s_i - s_j \ge \gamma$. To guarantee $\Delta_i > \Delta_j$, we need the linear term to overpower the remainder noise:
\begin{equation}
\eta (s_i - s_j) > |r_i - r_j|.
\end{equation}
Using the triangle inequality and the bound from Lemma 1:
\begin{align}
|r_i - r_j| &\le |r_i| + |r_j| \nonumber \\
&\le \frac{\beta \eta^2}{2} (\|\boldsymbol{g}_i\|^2 + \|\boldsymbol{g}_j\|^2).
\end{align}
Substituting the margin $\gamma$, the sufficient condition becomes:
\begin{equation}
\eta \gamma \ge \frac{\beta \eta^2}{2} (\|\boldsymbol{g}_i\|^2 + \|\boldsymbol{g}_j\|^2).
\end{equation}
Dividing by $\eta$ (since $\eta > 0$) and rearranging yields the bound for the learning rate:
\begin{equation}
\eta \le \frac{2\gamma}{\beta(\|\boldsymbol{g}_i\|^2 + \|\boldsymbol{g}_j\|^2)}.
\end{equation}
Thus, providing $\eta$ satisfies this condition (which is typical in fine-tuning where $\eta$ is small), the ranking based on the gradient inner product $s_i$ is theoretically consistent with the true utility $\Delta_i$.
\end{proof}

\subsection{Proof of Ranking Stability}
\label{app:proof_theorem_stability}

We analyze how the score $s_{i,t} = \boldsymbol{g}_{\mathcal{V}}(\boldsymbol{\theta}_t)^\top \boldsymbol{g}_i(\boldsymbol{\theta}_t)$ changes as the model parameters drift from $\boldsymbol{\theta}_0$ to $\boldsymbol{\theta}_t$.

\textbf{Lemma 2 (Score Stability).} \textit{Let parameter drift be bounded by $\|\boldsymbol{\theta}_t - \boldsymbol{\theta}_0\| \le R$. Under Lipschitz continuity assumptions with constants $L_{\mathcal{V}}, L_i$, the score change is bounded by:}
\begin{equation}
|s_{i,t}-s_{i,0}| \le (L_{\mathcal{V}}B_i + L_iB_{\mathcal{V}})R.
\end{equation}

\begin{proof}
Let $\boldsymbol{\delta} = \boldsymbol{\theta}_t - \boldsymbol{\theta}_0$. We decompose the score difference:
\begin{align}
s_{i,t} - s_{i,0} &= \boldsymbol{g}_{\mathcal{V}}(\boldsymbol{\theta}_t)^\top \boldsymbol{g}_i(\boldsymbol{\theta}_t) - \boldsymbol{g}_{\mathcal{V}}(\boldsymbol{\theta}_0)^\top \boldsymbol{g}_i(\boldsymbol{\theta}_0) \nonumber \\
&= \boldsymbol{g}_{\mathcal{V}}(\boldsymbol{\theta}_t)^\top (\boldsymbol{g}_i(\boldsymbol{\theta}_t) - \boldsymbol{g}_i(\boldsymbol{\theta}_0)) \nonumber \\
&\quad + (\boldsymbol{g}_{\mathcal{V}}(\boldsymbol{\theta}_t) - \boldsymbol{g}_{\mathcal{V}}(\boldsymbol{\theta}_0))^\top \boldsymbol{g}_i(\boldsymbol{\theta}_0).
\end{align}
Applying the Cauchy-Schwarz inequality and Lipschitz definitions:
\begin{align}
|s_{i,t} - s_{i,0}| &\le \|\boldsymbol{g}_{\mathcal{V}}(\boldsymbol{\theta}_t)\| \|\boldsymbol{g}_i(\boldsymbol{\theta}_t) - \boldsymbol{g}_i(\boldsymbol{\theta}_0)\| \nonumber \\
&\quad + \|\boldsymbol{g}_{\mathcal{V}}(\boldsymbol{\theta}_t) - \boldsymbol{g}_{\mathcal{V}}(\boldsymbol{\theta}_0)\| \|\boldsymbol{g}_i(\boldsymbol{\theta}_0)\| \nonumber \\
&\le B_{\mathcal{V}} \cdot (L_i \|\boldsymbol{\theta}_t - \boldsymbol{\theta}_0\|) + (L_{\mathcal{V}} \|\boldsymbol{\theta}_t - \boldsymbol{\theta}_0\|) \cdot B_i \nonumber \\
&= (B_{\mathcal{V}} L_i + L_{\mathcal{V}} B_i) \|\boldsymbol{\theta}_t - \boldsymbol{\theta}_0\|.
\end{align}
Substituting $\|\boldsymbol{\theta}_t - \boldsymbol{\theta}_0\| \le R$ yields the lemma.
\end{proof}

\textbf{Theorem 2 (Ranking Preservation).} \textit{If the initial margin satisfies $|s_{i,0}-s_{j,0}| > 2(L_{\mathcal{V}}(B_i+B_j)+(L_i+L_j)B_{\mathcal{V}})R$, then $s_{i,0}>s_{j,0} \implies s_{i,t}>s_{j,t}$.}

\begin{proof}
Let $M_0 = s_{i,0} - s_{j,0}$ be the initial margin. The deviation of the margin at step $t$ is:
\begin{equation}
(s_{i,t} - s_{j,t}) - M_0 = (s_{i,t} - s_{i,0}) - (s_{j,t} - s_{j,0}).
\end{equation}
Using Lemma 2, we bound the total deviation:
\begin{align}
|Deviation| &\le |s_{i,t} - s_{i,0}| + |s_{j,t} - s_{j,0}| \nonumber \\
&\le \underbrace{[L_{\mathcal{V}}(B_i + B_j) + B_{\mathcal{V}}(L_i + L_j)]}_{\text{Constant } K} R.
\end{align}
The condition states $M_0 > 2KR$. Thus, the worst-case margin at step $t$ is:
\begin{equation}
s_{i,t} - s_{j,t} \ge M_0 - KR > 2KR - KR > 0.
\end{equation}
Since the margin remains positive, the ranking order is preserved.
\end{proof}

\section{Cross-Model Transferability}
\label{app:proof_theorem_projection}

To justify the transferability of utility scores from a lightweight proxy to a large target model, we rely on the \textit{Subspace Alignment Hypothesis}. This framework posits that models trained on similar data manifolds share a common low-rank gradient structure, differing primarily by a rotation and scaling factor.

\paragraph{Gradient Subspace Alignment Assumption.}
Let $\boldsymbol{g}^{(m)}$ and $\boldsymbol{g}^{(M)}$ denote the gradient vectors of the proxy model $m$ and target model $M$, respectively. We assume there exists a linear isometry $\mathbf{P}$ (representing projection/rotation between feature spaces) and a global scaling factor $\alpha > 0$ such that the target gradients can be approximated by the proxy gradients with a bounded error $\epsilon$:
\begin{equation}
\|\boldsymbol{g}^{(M)} - \alpha \mathbf{P} \boldsymbol{g}^{(m)}\| \le \epsilon.
\end{equation}

\paragraph{Implication for Ranking Consistency.}
Under this assumption, the utility score on the target model $s_i^{(M)}$ can be analytically decomposed into the scaled proxy score and a residual noise term due to misalignment:
\begin{equation}
s_i^{(M)} \approx \alpha^2 s_i^{(m)} + \mathcal{R}_{\text{noise}},
\end{equation}
where $s_i^{(m)}$ is the utility calculated on the lightweight proxy. 

This relationship implies that the **ranking order is preserved** ($s_i^{(M)} > s_j^{(M)}$ implies $s_i^{(m)} > s_j^{(m)}$) provided that the \textit{utility margin} measured by the proxy is sufficiently large to dominate the projection error:
\begin{equation}
\underbrace{\alpha^2 (s_i^{(m)} - s_j^{(m)})}_{\text{Scaled Proxy Margin}} > \underbrace{|\mathcal{R}_{\text{noise}, i} - \mathcal{R}_{\text{noise}, j}|}_{\text{Alignment Noise}}.
\end{equation}
Therefore, OST is theoretically grounded to select high-utility data for the target model as long as the proxy maintains a structural gradient alignment (low $\epsilon$) with the target architecture.

\begin{figure*}[h]
    \centering
    \includegraphics[width=1.0\linewidth]{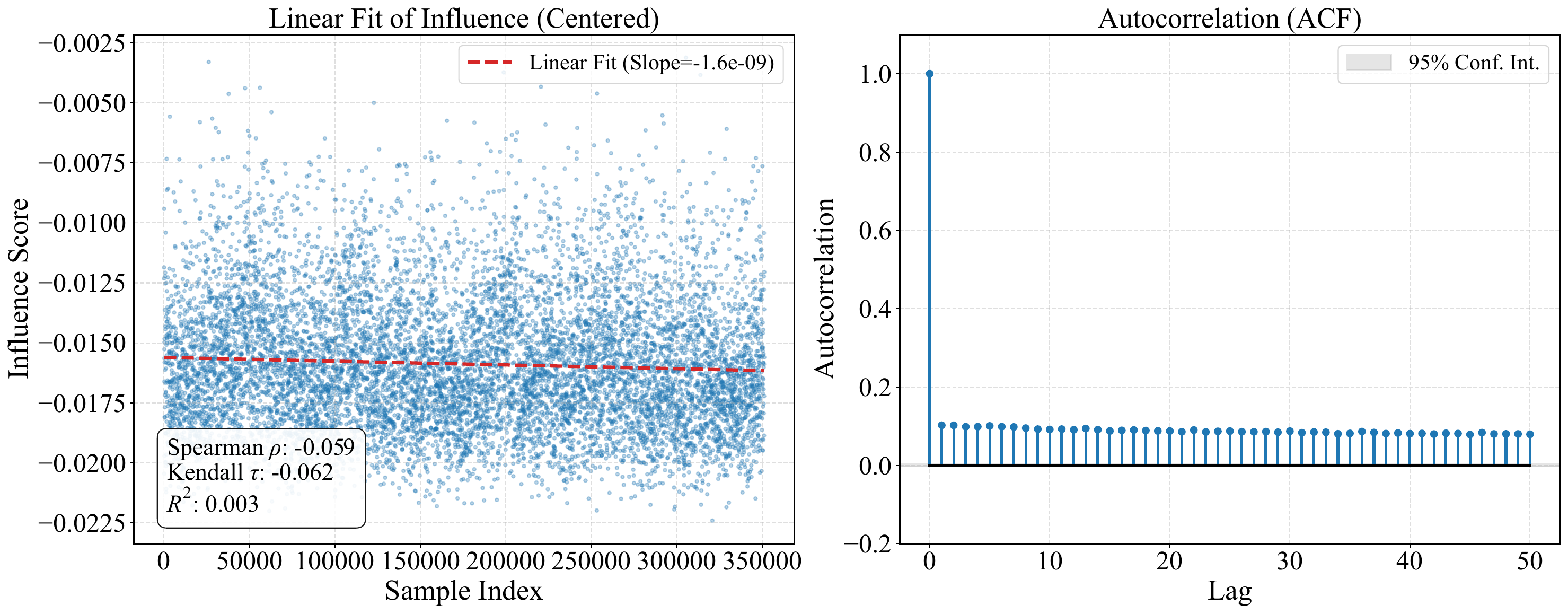}
    \caption{\textbf{Statistical Verification of Utility Independence.} \textit{Left}: The linear fit of influence scores against the sample index shows no significant temporal drift ($R^2 \approx 0.003$). \textit{Right}: The Autocorrelation Function (ACF) indicates no significant dependency between consecutive samples, validating the independent nature of the proxy scoring.}
    \label{fig:influence_analysis}
\end{figure*}

\section{Implementation Details}
\label{app:experimental_implementation}

\subsection{Data Construction Pipeline}
\label{sec:data_details}
Our training corpus originates from a proprietary collection of authentic mathematical problems acquired from professional data vendors. These resources are extensively sourced from standard textbooks and auxiliary teaching materials. To simulate realistic visual reasoning scenarios, we employed an \textbf{inverse rendering} technique: the textual question stems, geometric diagrams, and multiple-choice options were rendered together onto a single composite image. To bridge the reasoning gap in this raw visual data, we deployed a robust synthesis pipeline:
\begin{itemize}
    \item \textbf{CoT Synthesis:} We utilized \textbf{Doubao-seed-1.6-thinking} to generate step-by-step Chain-of-Thought (CoT) rationales (Prompt detailed in Figure \ref{fig:synthesis_prompt}). The generation process was configured with a temperature $T=0.3$, $top\_p=0.95$, and a maximum token limit of 16,384. Under these settings, the model was prompted to decompose complex logic into intermediate derivations before concluding with a final answer.
    
    \item \textbf{Quality Filtration:} To ensure high data fidelity, we applied a rigorous three-stage filtering protocol to the raw synthetic corpus: (1) \textbf{Format Verification} to ensure valid Markdown syntax and the presence of a \texttt{\textbackslash\textbackslash boxed \{\}} answer; (2) \textbf{Ground Truth Consistency} check, where samples are discarded if the synthesized answer deviates from the gold standard; and (3) \textbf{Length Constraint}: we observed that responses exceeding 10,000 tokens frequently exhibited degeneration issues (e.g., repetition loops), and thus excluded these samples to maintain data density.
\end{itemize}
The final filtered dataset $\mathcal{D}_{\mathrm{syn}}$ comprises 351,157 high-quality samples with an average length of 2,880 tokens.

\subsection{OST Utility Scoring (Phase 1)}
We employ \textbf{InternVL3-1B} as the lightweight computational proxy. The scoring pipeline is implemented on a node with 8 $\times$ NVIDIA H100 GPUs using \textbf{DeepSpeed ZeRO Stage 1} \cite{rajbhandari2020zeromemoryoptimizationstraining}. To overcome the I/O bottleneck inherent in the \textit{train-validate-revert} loop, we customized the \texttt{Accelerate} library to enable persistent in-memory caching of model parameters and optimizer states.

\begin{itemize}
    \item \textbf{Anchor Set ($\mathcal{V}$):} We constructed a representative validation anchor by stratified sampling 100 examples from a held-out pool of 500 problems. To minimize data loading latency during frequent evaluation, these samples are pre-loaded and pinned in CPU memory.
    
    \item \textbf{Gradient Simulation:} For each candidate $z_i$, we simulate a single optimization step using SGD (LR $1\text{e-}7$, batch size 1). We measure utility $\Delta_i$ as the reduction in cross-entropy loss on $\mathcal{V}$ immediately following this look-ahead update. Thanks to the state-caching mechanism, we achieved a high throughput of \textbf{0.5--1.0 seconds} per sample in \texttt{bfloat16} precision.
    \item \textbf{Independence Verification:} To ensure that the utility scores ($\Delta_i$) reflect the intrinsic quality of the data rather than artifacts of the sequential processing (e.g., curriculum effects or state drift), we conducted a statistical independence test. As shown in Figure~\ref{fig:influence_analysis}, the scores exhibit a negligible linear trend (Slope $\approx -1.6 \times 10^{-9}$, $R^2 \approx 0.003$) and near-zero autocorrelation across lags. This confirms that the scoring process is statistically stationary and robust to permutation.

\end{itemize}

\subsection{Target Model Training (Phase 2)}
We evaluate the selected subsets on two distinct target architectures using standard Supervised Fine-Tuning (SFT).

\paragraph{Internal 30B Multimodal Model.}
This model integrates a Qwen3-30B-A3B language backbone with a ViT-300M vision encoder. Training is conducted on a cluster of \textbf{64 $\times$ NVIDIA H100 GPUs} employing 8-way Tensor Parallelism (TP) and Pipeline Parallelism (PP). To prevent catastrophic forgetting, we adopt a data mixing strategy where the selected math data is combined with general text corpora (sampling weight 0.25) and multimodal interleaved data (weight 0.025). Optimization uses the \textbf{AdamW} optimizer with $\beta_1=0.9, \beta_2=0.95$, a global batch size of 64, and a peak learning rate of $1\text{e-}5$ with cosine decay.

\paragraph{Open-Source Qwen3-VL Series.}
To verify generalization, we fine-tune the Qwen3-VL-2B, 4B, and 8B Instruct models. Experiments run on a single node with \textbf{8 $\times$ NVIDIA H100 GPUs} using \textbf{DeepSpeed ZeRO Stage 3}. We freeze the vision encoder and only update the LLM backbone. Hyperparameters include a learning rate of $1\text{e-}7$, a global batch size of 32 (16 for the 8B model), and a maximum sequence length of 16,384 tokens.

\subsection{Evaluation and Baselines}
\paragraph{Baselines.}
We compare OST against: (1) \textbf{Full SFT}, using the entire dataset to establish a performance ceiling/floor; (2) \textbf{Random Selection}, providing a lower bound; and (3) \textbf{LLM-as-a-Judge}, the industry standard where \textbf{Qwen3VL-235B-A22B-Instruct} acts as a verifier to perform rejection sampling based on reasoning correctness (Prompt detailed in Figure \ref{fig:judge_prompt}).

\paragraph{Inference Protocol.}
We adopt a dual-track evaluation covering our 500-sample \textbf{Internal Benchmark} (10 sub-domains) and four external benchmarks (\textbf{MathVision}, \textbf{MathVista}, \textbf{WeMath}, \textbf{LogicVista}). Inference is accelerated via \textbf{vLLM} \cite{kwon2023efficientmemorymanagementlarge} with temperature $T=0.6$, $top\_p=0.95$, and a max token limit of 32,768. Open-ended responses are scored by \textbf{GPT-4o-mini} \cite{openai2024gpt4ocard} to ensure impartial evaluation.

\section{Case Study}
\label{app:Case_Study}

We present a comparative analysis focusing on the integrity of multimodal data and the reliability of automated verification. To illustrate the distinction between toxic noise and high-utility data, we pair bad samples (failures/bias, ranking bottom 5\%) with good samples (strong alignment, ranking top 5\%) in Figures \ref{fig:pair_recognition}, \ref{fig:pair_context}.

\section{Prompts}
\label{app:prompts}

This section details the specific instructions used for data synthesis and the baseline comparison. Figure \ref{fig:synthesis_prompt} illustrates the prompt used to generate step-by-step reasoning chains, while Figure \ref{fig:judge_prompt} displays the verification instruction used for the LLM-as-a-Judge baseline.

% ==========================================
% Pair 1: Basic Recognition vs. Conceptual Grounding
% ==========================================
\clearpage
\begin{figure*}[t!]
    \centering
    % --- Negative: Recognition Failure ---
    \begin{minipage}{.85\textwidth}
        \centering
        \includegraphics[width=1.0\linewidth]{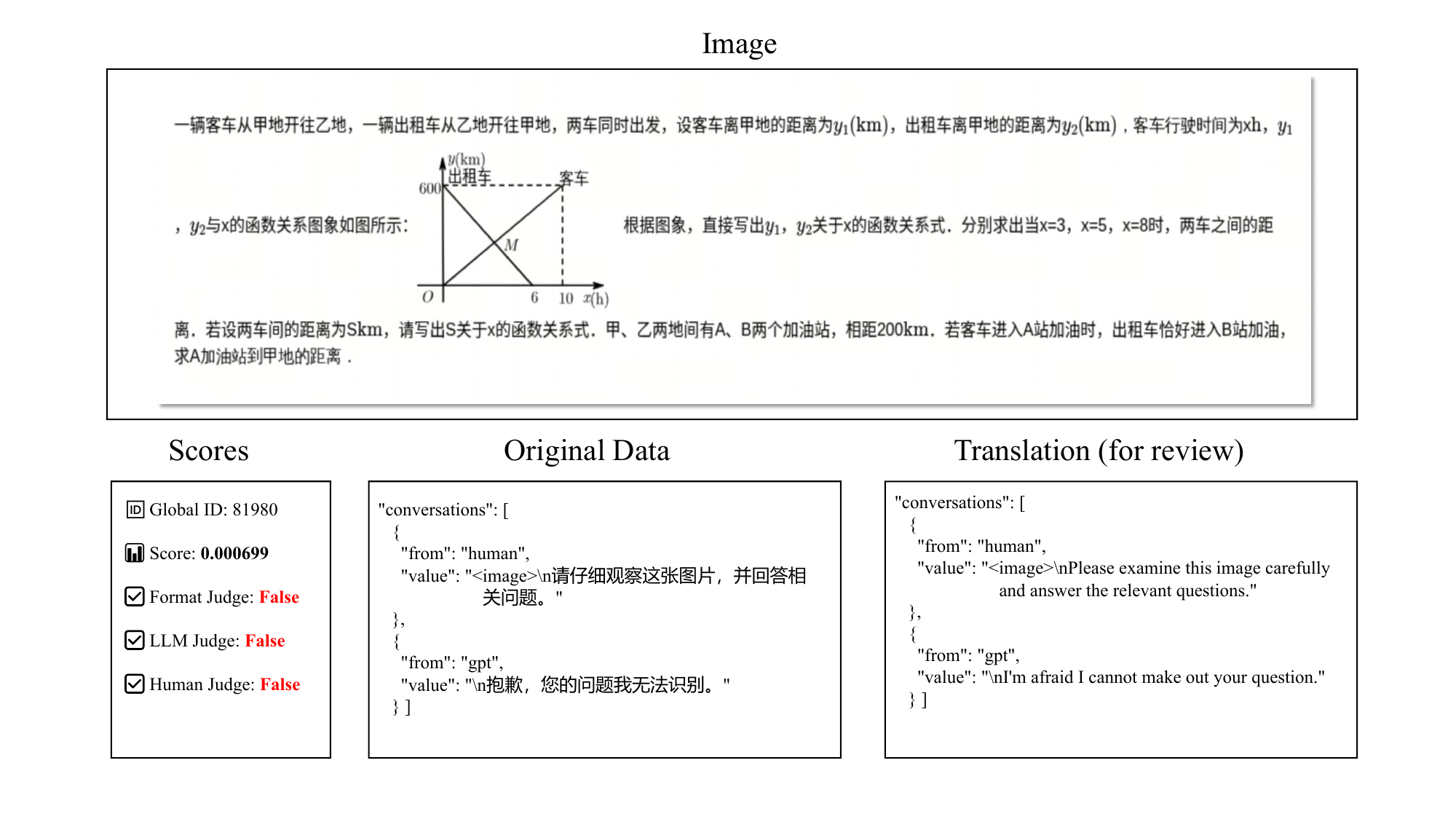} % ID: 81980
        \caption*{\textbf{(a) Bad Sample: Visual Recognition Failure (Refusal)} \\ 
        \textit{Judge: False.} The model fails to extract features from the chart, triggering a safety refusal. Training on such data dilutes instruction-following capabilities.}
    \end{minipage}
    
    \vspace{1.5em} % Spacing between the pair
    
    % --- Positive: Conceptual Grounding ---
    \begin{minipage}{.85\textwidth}
        \centering
        \includegraphics[width=1.0\linewidth]{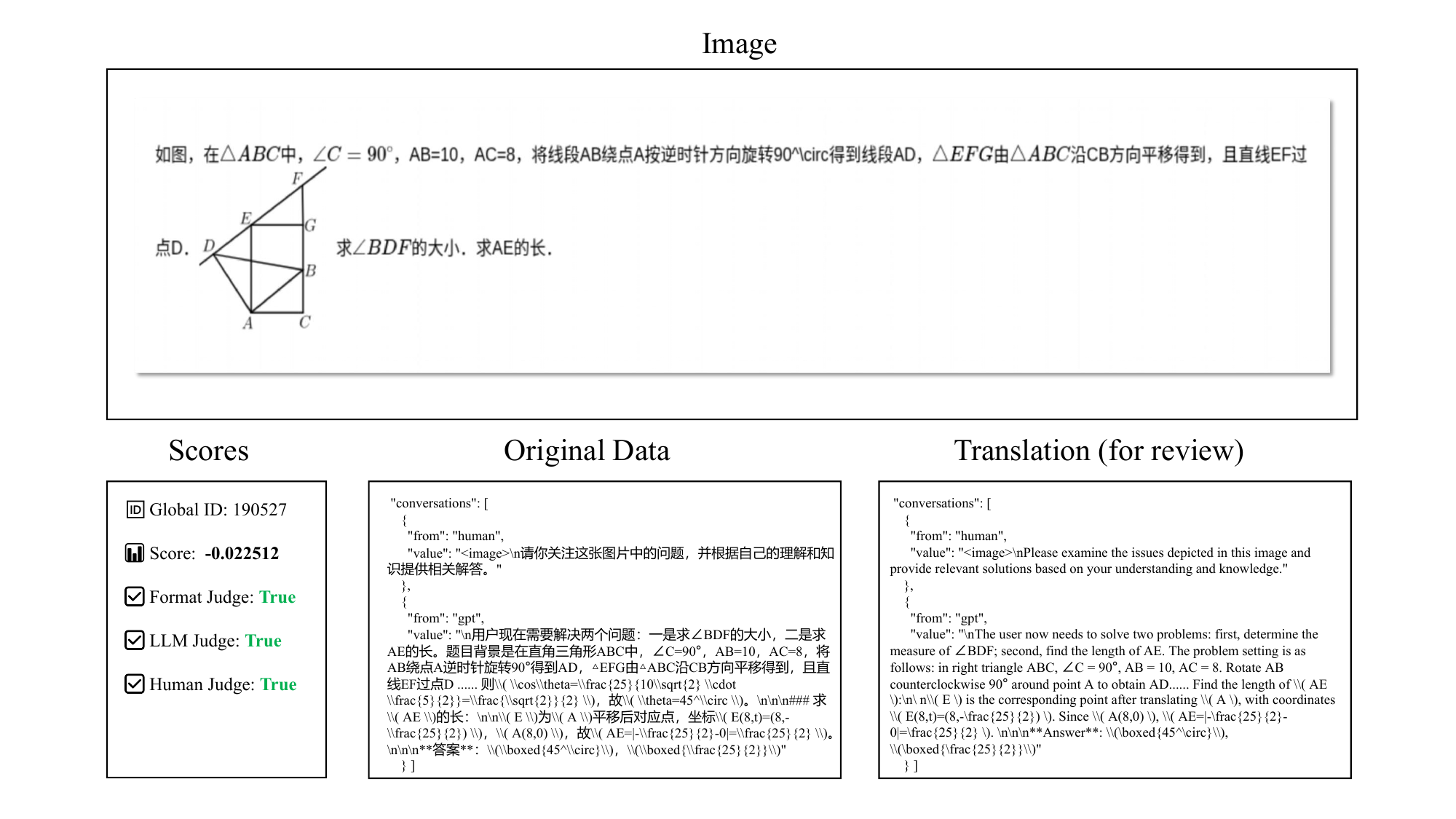} % ID: 236568
        \caption*{\textbf{(b) Good Sample: Conceptual Grounding} \\ 
        \textit{Judge: True.} In contrast, this sample correctly translates visual abstractions (geometric slope) into symbolic equations without hallucination.}
    \end{minipage}
    
    \caption{\textbf{Comparison Group 1: Basic Perception.} While the negative sample (a) fails at the pixel-to-text stage, the positive sample (b) demonstrates successful grounding of abstract visual concepts.}
    \label{fig:pair_recognition}
\end{figure*}

% ==========================================
% Pair 2: Missing Info vs. Spatial Reasoning
% ==========================================
\clearpage
\begin{figure*}[t!]
    \centering
    % --- Negative: Information Loss ---
    \begin{minipage}{.85\textwidth}
        \centering
        \includegraphics[width=1.0\linewidth]{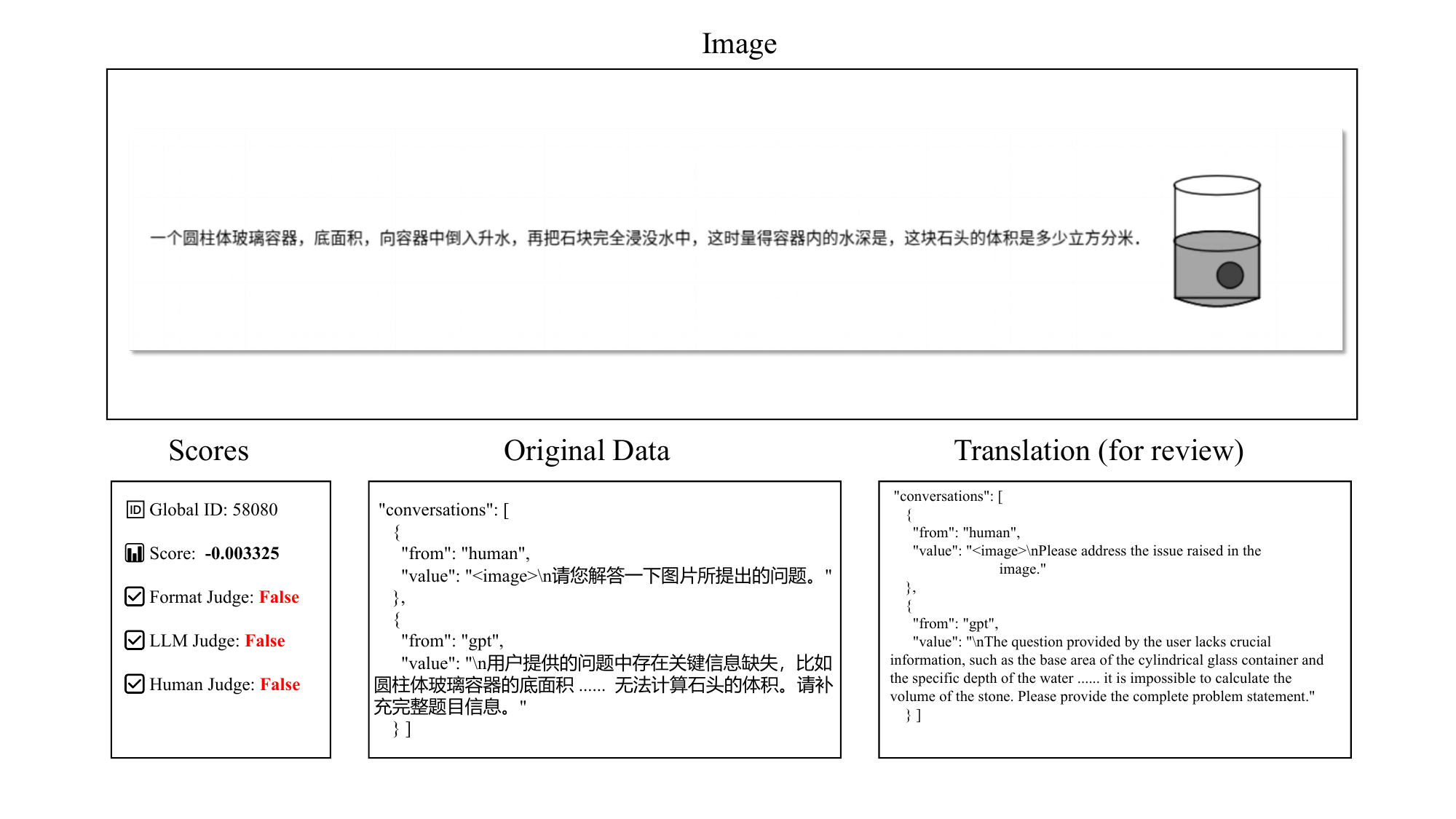} % ID: 58080
        \caption*{\textbf{(c) Bad Sample: Contextual Information Loss} \\ 
        \textit{Judge: False.} Critical conditions (e.g., dimensions) are missing due to cropping/OCR errors. This represents dead-end noise.}
    \end{minipage}
    
    \vspace{1.5em}
    
    % --- Positive: Spatial Reasoning ---
    \begin{minipage}{.85\textwidth}
        \centering
        \includegraphics[width=1.0\linewidth]{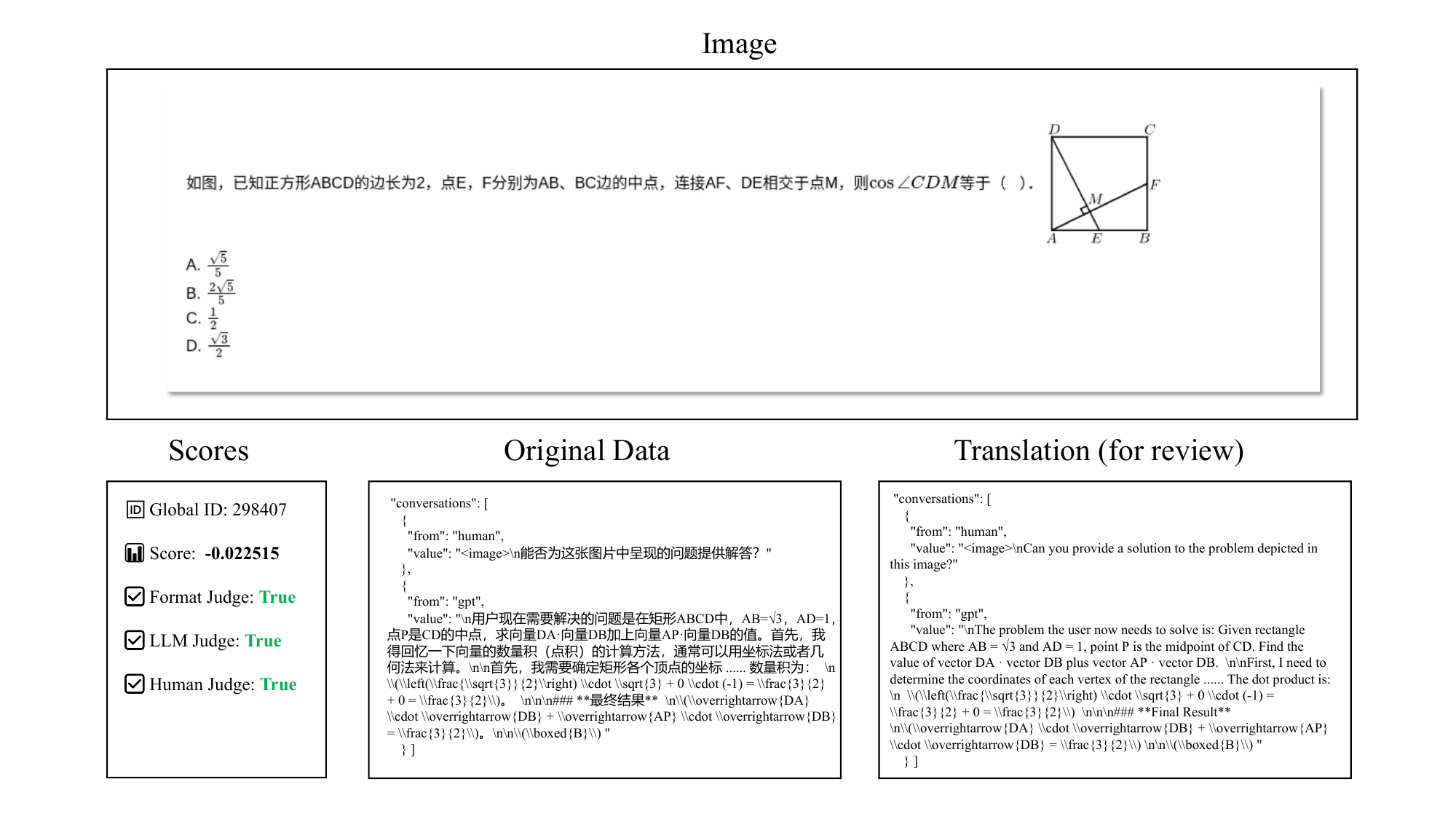} % ID: 190527
        \caption*{\textbf{(d) Good Sample: Spatial Transformation Consistency} \\ 
        \textit{Judge: True.} A high-utility sample requiring multi-step spatial tracking (rotation/translation). The derivation aligns perfectly with visual constraints.}
    \end{minipage}
    
    \caption{\textbf{Comparison Group 2: Context and Reasoning.} Contrast between incomplete data that prevents reasoning (c) and data rich in spatial constraints (d) that provides strong gradient signals.}
    \label{fig:pair_context}
    
\end{figure*}

\begin{figure*}[h]
    \centering
    
    % --- Top Image: Synthesis Prompt ---
    % 将宽度调整为 0.95\linewidth 或 1.0\linewidth 以利用页面宽度
    \includegraphics[width=0.95\linewidth]{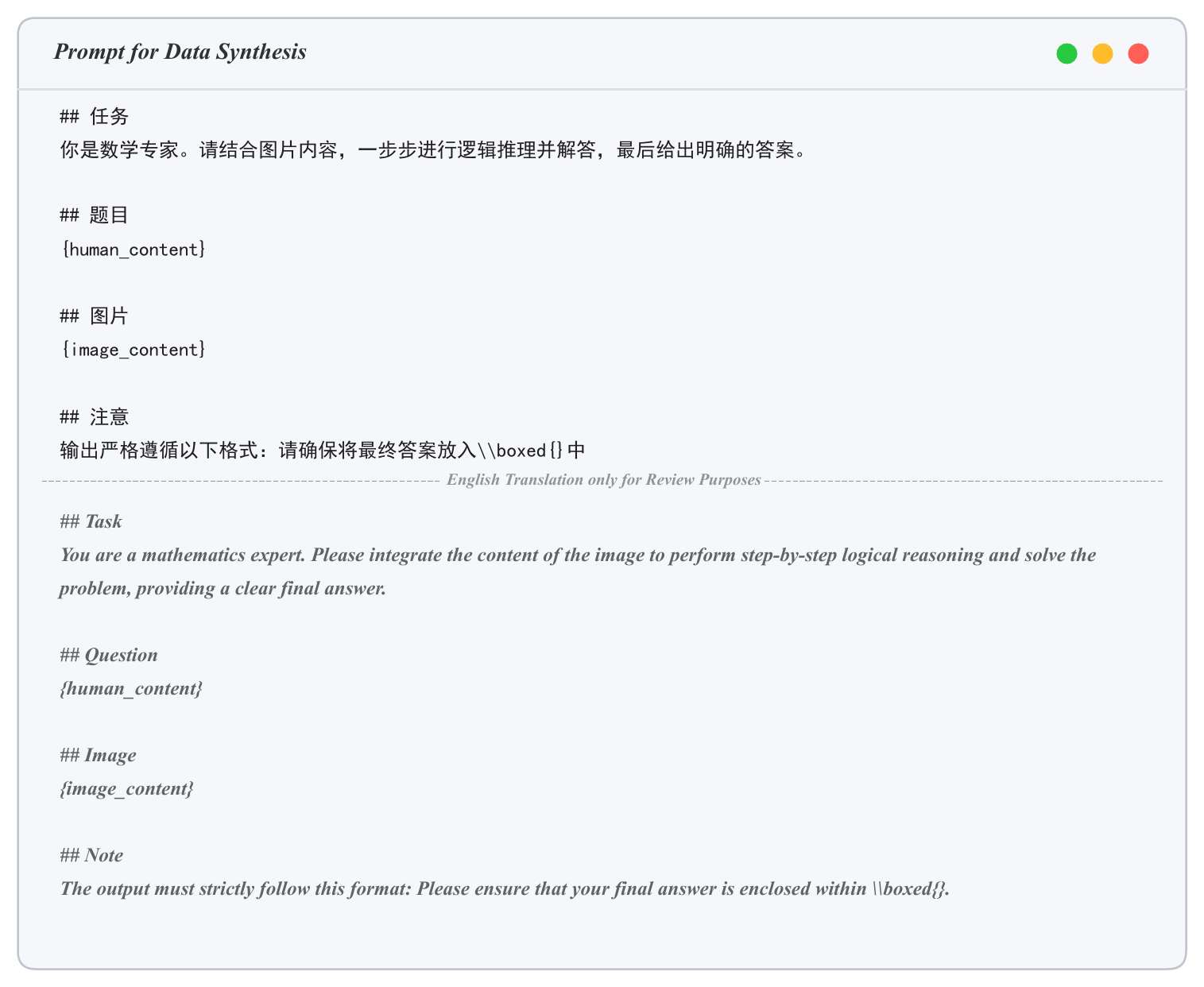}
    \vspace{-0.5em}
    \caption{\textbf{Prompt for CoT Synthesis.} The instruction template used to generate reasoning traces with Doubao-Seed-1.6-thinking during data construction.}
    \label{fig:synthesis_prompt}
\end{figure*}

\begin{figure*}[h]
    % --- Bottom Image: Judge Prompt ---
    \includegraphics[width=0.75\linewidth]{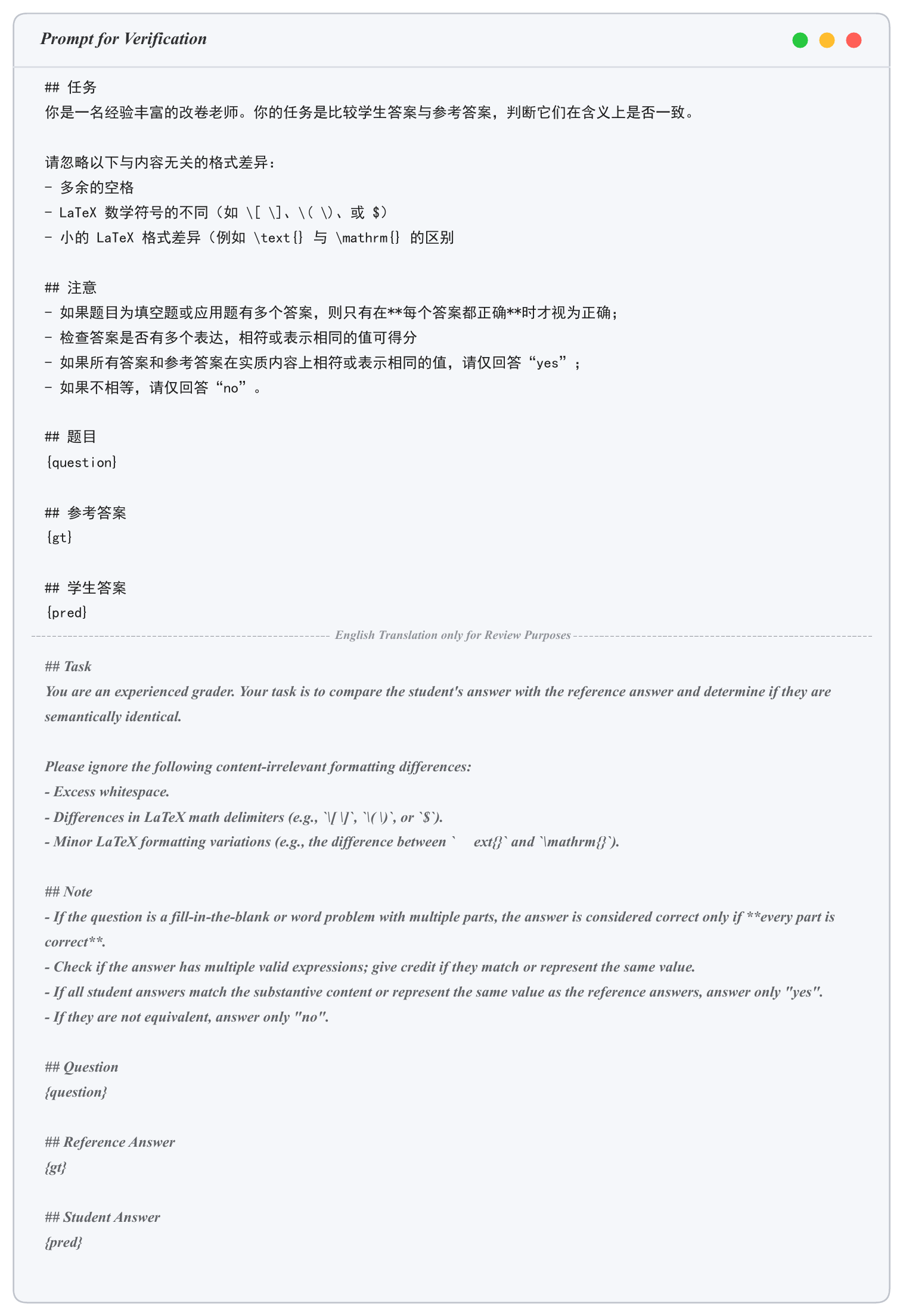}
    \caption{\textbf{Prompt for LLM-as-a-Judge.} The verification instruction template used by Qwen3VL-235B-A22B-Instruct for the baseline rejection sampling.}
    \label{fig:judge_prompt}
    
\end{figure*}

\end{document}